\documentclass[a4paper,10pt]{amsart}

\usepackage[top=2.5cm, bottom=2.5cm, left=2.5cm, right=2.5cm]{geometry}

\usepackage{amssymb,amsmath,amsthm,ascmac,array,bm}
\usepackage[dvipdfmx]{graphicx}
\usepackage{color}
\usepackage[nocompress]{cite}
\usepackage{times}
\usepackage{amsbsy}
\usepackage{graphicx}
\usepackage{algorithm}
\usepackage{algorithmic}
\usepackage{epstopdf}
\usepackage{bbold}
\usepackage{float} 
\usepackage{rotating}
\usepackage{setspace}
\usepackage{cite}
\usepackage{lscape}
\usepackage{multirow}
\usepackage{multicol}
\usepackage{url}
\usepackage{MnSymbol}
\usepackage{color,soul}
\usepackage{array}
\newtheorem{thm}{Theorem}

\newtheorem{definition}{Definition}
\newtheorem{proposition}{Proposition}
\newtheorem{corollary}{Corollary}
\usepackage{caption}
\usepackage{subcaption}
\usepackage{chngpage}
\usepackage{dsfont}
\usepackage{yfonts}
\usepackage{enumitem}
\usepackage[makeroom]{cancel}

\title[Towards Competitive Classifiers for Unbalanced Classification Problems]
{Towards Competitive Classifiers for Unbalanced Classification Problems: A Study on the Performance Scores}

\author{Jonathan~Ortigosa-Hern\'andez}
\author{I\~naki~Inza}
\author{~Jose~A.~Lozano}
\thanks{{\it Institutions.} The authors are with the Intelligent Systems Group at the Department
of Computer Science and Artificial Intelligence, Computer Science Faculty, The University of the Basque Country UPV/EHU, Donostia-San Sebasti\'an, Spain.Moreover, Jose A. Lozano is also with the Basque Center for Applied Mathematics BCAM, Bilbao, Spain.\\
{\it Corresponding author.} jonathan.ortigosa@ehu.es}

\date{}

\begin{document}
\setlength{\baselineskip}{4.5mm}

\begin{abstract}
Although a great methodological effort has been invested in proposing competitive solutions to the class-imbalance problem, little effort has been made in pursuing a theoretical understanding of this matter.

In order to shed some light on this topic, we perform, through a novel framework, an exhaustive analysis of the adequateness of the most commonly used performance scores to assess this complex scenario. We conclude that using unweighted H\"older means with exponent $p \leq 1$ to average the recalls of all the classes produces adequate scores which are capable of determining whether a classifier is competitive.

Then, we review the major solutions presented in the class-imbalance literature. Since any learning task can be defined as an optimisation problem where a loss function, usually connected to a particular score, is minimised, our goal, here, is to find whether the learning tasks found in the literature are also oriented to maximise the previously detected adequate scores. We conclude that they usually maximise the unweighted H\"older mean with $p = 1$ (a-mean).

Finally, we provide bounds on the values of the studied performance scores which guarantee a classifier with a higher recall than the random classifier in each and every class.
\end{abstract}

\maketitle


\section{Introduction}\label{sec:intr}

In many classification problems, there are significant differences among the probabilities of the classes, i.e. the probability of a particular example belonging to a certain class. In the literature, this situation is known as the class-imbalance problem \cite{He09} and it is considered to be a major obstacle to building competitive classifiers. By {\it competitive classifiers} we refer to classifiers showing not only a high overall classification error, but also a balance between the prediction powers for both the most and the least probable classes.

Moreover, the investigated class-imbalance problem has somewhat torn apart the conventional approaches to solve classification problems. In traditional supervised learning approaches \cite{Wu07}, classification accuracy is by far the most popular numerical performance score due to its theoretical foundations, its simplicity and its property of being intuitive \cite{Pro98}. On one hand, it has been used to guide learning processes due to the fact that most learning algorithms are designed to asymptotically converge to the behaviour of the Bayes decision rule \cite{You74}, a classifier which is optimal for this score. On the other hand, this performance score has also been broadly used in the literature to evaluate the performance of real-world classifiers \cite{Pro98}. However, in recent years, the research community has noticed that it is not an adequate score for problems with extremely skewed class distributions; the least probable classes have very little impact on the classification accuracy when compared to the most probable classes \cite{Gu09}. When dealing with highly unbalanced problems, this implies that (i) traditional learning approaches maximising the classification accuracy may produce dummy classifiers which always predict the most probable classes \cite{Dru05}, and that (ii) this performance score is no longer convenient for assessing real-world classifiers since a high value does not guarantee a fair prediction power for the underrepresented classes.

In view of this puzzling situation, there has been significant methodological effort invested not only in determining which performance scores are the most appropriate to the class-imbalance scenario \cite{Gu09}, but also in proposing new learning systems to obtain competitive classifiers for highly unbalanced classification problems \cite{Fer11}\cite{Wan12}\cite{Lop13}. However, little theoretical effort \cite{Dru05}\cite{Wei04} has been made in pursuing a complete understanding of these topics. Most of the successful proposals are the result of experience, systematic studies \cite{Jap02} or just pure intuition on how special prominence can be given to the least probable classes \cite{Wan12}, rather than being built on a solid theoretical foundation \cite{He09}. To the best of our knowledge, the following questions are still unanswered in the literature:

\begin{enumerate}
\item Which performance scores are adequate to determine the competitiveness of a classifier in unbalanced domains?
\item Which performance scores (loss functions) are maximised (minimised) in the most common learning solutions designed to deal with skewed classes?
\item Can bounds guaranteeing the competitiveness of a classifier be provided for certain adequate performance scores?
\end{enumerate}

\noindent Thus, in order to shed some theoretical light on the previous three questions, in this paper, we perform the following three studies:

{\bfseries A first study on determining the adequate performance scores for class-imbalance scenarios.} Our main objective is to analyse how different performance scores behave under a changing class distribution when the Bayes decision rule is used as a classifier. The usage of this rule is due to the fact that, as it has been long studied in the literature e.g. \cite{Dru05}, there is a complete understanding of the deficiencies of this classifier in this complex scenario. Hence, this knowledge can be exploited in order to obtain a licit characterisation of the adequateness of the performance scores; by having a look at the fluctuation of the values of the performance scores when the class distribution varies, we can devise which scores are the most appropriate to evaluate the classifiers for unbalanced domains. However, having differences among the probabilities of the classes is not the only factor hindering the proposed solutions \cite{Lop13}; other factors such as class-overlapping or the presence of noise data also modify the values of the performance scores. Thus, based on these grounds, we define a novel controlled framework where these other hindering factors can be properly cancelled so that the contribution of the class distribution to the performance scores can be legitimately quantified in isolation. Conclusions show that the performance scores which are unweighted H\"{o}lder means \cite{Bul03} with $p \leq 1$ among the recalls (Section \ref{sec:not}) are the most appropriate to evaluate the competitiveness of classifiers in unbalanced problems. In this regard, misclassifying the least probable classes is highly penalised and this penalisation is increased as the value of $p$ decreases.

{\bfseries A second study on discovering the performance scores maximised in the major solutions of the class-imbalance literature.}  Since it is known that the typical supervised learning techniques are unable to deal with the class-imbalance problem \cite{Dru05} and given the large number of publications reporting positive results \cite{Bat04}\cite{Liu06}, we analyse the most common approaches to the class-imbalance problem by assuming that different performance scores than classification accuracy are maximised in those solutions. Our main conclusion is that most of those proposals output classifiers which are maximised for the unweighted H\"{o}lder mean among the recalls with $p = 1$ (a-mean) and that they have an asymptotic behaviour close to the Bayes decision rule for an equiprobable class distribution (an optimal rule for that performance score). Moreover, we also show that, in scenarios showing skewed class distributions, this rule outperforms the well-known Bayes decision rule in terms of the values of the adequate scores detected in the first study. Yet, the usage of any other unweighted H\"{o}lder mean with $p<1$ to define new learning procedures for dealing with unbalance problems is also appropriate.

{\bfseries A third study on proposing bounds for the performance scores to determine the competitiveness of classifiers in unbalanced domains.} For this purpose, we delimit the definition of competitiveness of a given solution to a classifier having higher recalls than the random classifier in each and every class. Thus, we can provide two different practical bounds for the values of the performance scores expressed as unweighted H\"{o}lder means among the recalls with $p  \in \mathds{R} \cup \{+\infty,-\infty\}$; a bound for the lowest value of the performance score ensuring a competitive solution and a bound for the highest value of the score indicating an incompetent solution, i.e. the random classifier obtains a better recall in at least one class value. Here, our conclusions are also consistent with the first study; since the distance between both bounds decreases along with $p$ (rapidly for $p\leq1$), using H\"{o}lder means with $p\leq1$ is presumed to be adequate for determining the competitiveness of a classifier. In fact, both bounds coincide in the extreme $p=-\infty$.

The rest of the paper is organised as follows: 
First, Section \ref{sec:not} introduces the general definitions and notation. 
Then, Section \ref{sec:st1}, Section \ref{sec:st2}, and Section \ref{sec:st3} expose the above mentioned first, second, and third studies, respectively.
Finally, Section \ref{sec:sum} sums up the paper and discusses the potential lines of future work.

\section{General Definitions and Notation}\label{sec:not}

\subsection{H\"{o}lder Means}

As we make use of the H\"{o}lder means \cite{Bul03}, a.k.a. generalised means or power means, throughout the whole paper, we start by defining them as follows:
\begin{definition} Let $\mathbf{a}=(a_1,\ldots, a_n)$ be a series of $n$ positive real numbers with non-negative weights $\boldsymbol{\zeta}=(\zeta_1,\ldots, \zeta_n)$ s.t. $\sum_i \zeta_i=1$, then, a H\"{o}lder mean is a mean of the form
\begin{equation}
\label{holdermean}
M_p(\mathbf{a},\boldsymbol{\zeta}) = \left(\sum_{i=1}^n\zeta_ia_i^p\right)^{\frac{1}{p}}.
\end{equation}
\noindent Here, $p  \in \mathds{R} \cup \{+\infty,-\infty\}$ is an affinely extended real number. This family of functions has several interesting properties: 
\begin{enumerate}
\item {\bfseries H\"{o}lder mean inequality:} $M_p(\mathbf{a},\boldsymbol{\zeta}) \geq M_q(\mathbf{a},\boldsymbol{\zeta})$ if $p > q$. The equality only holds for the case of $a_i=a_j,\forall i,j$. 
\item {\bfseries Inclusion of the Pythagorean means\footnote{The classical definition of these means can be found in Table \ref{scores}.}:}  The arithmetic mean corresponds to the case $p=1$, the harmonic mean to $p=-1$, and the geometric mean is the limit of mean with an exponent $p$ approaching to $0$, i.e. $\lim_{p \to 0} M_p(\mathbf{a},\boldsymbol{\zeta})$.
\item {\bfseries Extremes:} $\lim_{p \to +\infty} M_p(\mathbf{a},\boldsymbol{\zeta}) = \max \{\zeta_1a_1,\ldots, \zeta_na_n\}$ and $\lim_{p \to -\infty} M_p(\mathbf{a},\boldsymbol{\zeta}) = \min \{\zeta_1a_1,\ldots, \zeta_na_n\}$.
\end{enumerate}

\end{definition}

\subsection{Unbalanced $K$-class Classification Problem}

Let $\gamma_K$ be a $K$-class classification problem with a generative model given by the generalised joint probability density function

\begin{equation}
\rho(\mathbf{x},c|\boldsymbol{\theta}) = p(c)\rho(\mathbf{x}|c,\boldsymbol{\theta}).
\label{genModel}
\end{equation}
\noindent Under the assumption of belonging to a given family of probability distributions, let $\rho(\mathbf{x}|c,\boldsymbol{\theta})$ be the conditional distribution of the feature space and let $\boldsymbol{\theta}=(\boldsymbol{\theta}_1,\boldsymbol{\theta}_2,\ldots,\boldsymbol{\theta}_K)$ stand for the set of the model parameters which unambiguously determine the conditional probability distributions. Here, each $\boldsymbol{\theta}_i$ represents the set of parameters for the distribution of each class $c_i$.
Also, let $p(c)$ be the distribution of the class probabilities. For simplicity of notation, henceforth, we denote $p(c)$ by  $\boldsymbol{\eta}=(\eta_1,\eta_2,\ldots,\eta_K)$, where each $\eta_i=p(c_i)$ is the probability of the categorical class $c_i$. Additionally, for convenience, hereafter, the special case of equiprobability, i.e. $\forall i, \eta_i = 1/K$, is denoted by $\mathbf{e}$. Therefore, according to \cite{He09}, a $K$-class classification problem is balanced if it exhibits a uniform class distribution. Otherwise, it is considered to be unbalanced. Formally, 
$$\gamma_K \textrm{ is balanced }  \Longleftrightarrow \boldsymbol{\eta} = \mathbf{e}.$$
When the model is unbalanced, in the multi-class scenario ($K>2$), we can differentiate two major types of class-imbalance \cite{Wan12}: (i) multi-majority unbalanced problems, i.e. when most of the classes have a higher probability than equiprobability, and (ii) multi-minority unbalanced problems, i.e. when most of the class probabilities are below the equiprobability. Formally, let $\{\mathbb{M},\mathbb{m}\}$ be a partition of the set $\{1,\ldots, K\}$ such that, $\forall{i \in \mathbb{M}}, \eta_i \geq 1/K$ (overrepresented) and $\forall{j \in \mathbb{m}}, \eta_j < 1/K$ (underrepresented), then
\begin{eqnarray*}
\gamma_K \textrm{ is multi-majority } \Longleftrightarrow  |\mathbb{M}| \geq K/2, &\textrm{ and } \gamma_K \textrm{ is multi-minority }\Longleftrightarrow  |\mathbb{m}| > K/2.
\end{eqnarray*}
By means of this definition, it can be presumed that having a balanced scenario is a hard condition to ensure in real-world problems. Here, $\boldsymbol{\eta}$ usually differs from equiprobability. For that reason, it is imperative to theoretically study not only the impact of the class distribution on the competitiveness of the proposed solutions and which performance scores are able to measure that impact, but also define more adequate learning systems which can effectively learn from highly skewed class distributions, i.e. $\exists i$ s.t. $(\eta_i\sim1 \vee \eta_i\sim0).$

\subsection{Traditional Supervised Learning Approaches}

Solving a $K$-class classification problem, regardless of its imbalance extent, is equivalent to learning a function $\Psi$, known as a classifier, that maps a vector of observations $\mathbf{x}$, drawn from the generative function of eq. (\ref{genModel}), into a categorical class $c_i$. In the supervised learning approach, the learning process is carried out by means of an optimisation algorithm, which provided with some labelled training data, drawn also from eq. (\ref{genModel}), attempts to infer a function $\Psi$ that minimises a certain loss function \cite{Bar06}. Here, most of the traditional learning algorithms use the $0$-$1$ loss or a surrogate loss which, by providing an upper bound for it, is also expected to minimise the $0$-$1$ loss \cite{Kan13}. This loss is also referred to as the misclassification error which can, in fact, be directly calculated by $(1 - \text{classification accuracy})$. This implies that those algorithms inherently maximise the classification accuracy and, therefore, they should asymptotically obtain classifiers close to the Bayes decision rule, a classifier which always obtains the highest classification accuracy for every $K$-class classification problem. In consequence, in this paper, we assume a framework where the generative function is known so that the Bayes decision rule \cite{You74} can be directly used as a representative of the classifiers resulting from the traditional approaches. By means of this approach, we can make use of the knowledge of prior works on the deficiencies of the traditional approaches in solving unbalanced problems \cite{Dru05}\cite{Axe00}\cite{Pra04} to complement our studies on class-imbalance.

\begin{definition} Assuming $\rho(\mathbf{x}|c,\boldsymbol{\theta})$ and $\boldsymbol{\eta}$ to be known, the {\it Bayes decision rule} (\textsf{BDR}) is given by
\begin{equation}
\label{BDR}
 \hat{c}_B = \arg\max_{i}\eta_i\rho(\mathbf{x}|c_i,\boldsymbol{\theta}_i).
\end{equation}
\noindent Here, $\hat{c}_B$ is the categorical class assigned by the \textsf{BDR} to the observation $\mathbf{x}$. This rule has a corresponding probability of error
\begin{equation}
\label{eB}
 e_B = 1 - \sum_{i=1}^K\eta_i\int\limits_{\Omega_i}\rho(\mathbf{x}|c_i,\boldsymbol{\theta}_i)d\mathbf{x}
\end{equation}
\noindent which is called the Bayes error and is the highest lower bound of the probability of error of any classifier. Here,
\begin{equation}
\label{region}
\Omega_i = \{\mathbf{x}: \eta_i\rho(\mathbf{x}|c_i,\boldsymbol{\theta}_i) - \max_{{i'} \neq i}\eta_{i'}\rho(\mathbf{x}|c_{i'},\boldsymbol{\theta}_{i'}) >0\}
\end{equation}
\noindent is the region where $\eta_i\rho(\mathbf{x}|c_i,\boldsymbol{\theta}_i)$ is maximum and, so, the instances are assigned to the class $c_i$ by the \textsf{BDR}, for all $i$.
\end{definition}

\subsection{Performance Scores}

Classifiers often produce misclassifications, and optimal classifiers are not exceptions. Thus, once a classifier is constructed, its associated discerning skill needs to be measured. When the generative function is assumed to be known, a common tool used for visualising the performance of a classifier is the true confusion matrix of a given classifier $\Psi$ \cite{Koc13}. It is a square matrix of size $K$ containing the mathematical expectations of classifying, with a classifier $\Psi$, an example of class $c_i$ (rows) as class $c_j$ (columns). Therefore, formally, the {\it true confusion matrix}\footnote{In most the real world applications, where $\rho(\mathbf{x},c)$ is unknown, an estimation of the true confusion matrix is utilised instead. It is usually referred to as the empirical confusion matrix, or simply as the confusion matrix.}  of the \textsf{BDR} is defined as $\mathbf{A}_{B}=[a_{i,j}]_{1 \leq i,j \leq K}$, where each element $a_{i,j}$ is calculated by:
\begin{equation}
a_{i,j}= \mathbb{E}_{\mathbf{x}|c=c_i} [\hat{c}_{B}=c_j]=\eta_i\int\limits_{\Omega_j}{}\rho(\mathbf{x}|c_i,\boldsymbol{\theta}_i)d\mathbf{x}.
\label{tcmbayes}
\end{equation}
\noindent Here, $\mathbb{E} [\cdot]$ stands for the mathematical expectation and $\Omega_j$ is defined as in eq. (\ref{region}). It can be easily noticed that, assuming that the family of distributions for the feature space is fixed, the calculation of the true confusion matrix depends on only three parameters; the used classifier ($\Psi=B$, the \textsf{BDR} is assumed here), the class distribution ($\boldsymbol{\eta}$) and the parameters of the generative function ($\boldsymbol{\theta}$).

As stated in \cite{Koc13}, the  confusion matrix is one of the most informative performance summaries that a multi-class learning system can rely on. Among other information, it contains how much the classifier is accurate on each class, and the way it tends to get confused among classes. However, it is often tedious not only determining the overall behaviour of a classifier from the confusion matrix, but also comparing among several classifiers.  Therefore, quantity measures which summarise the confusion matrix are often preferred. These measures are known as performance scores \cite{San15}. Since {\it the best behaviour} may vary from one kind of problem to another, there are many different and diverse performance scores in the community. This diversity may, at times, obscure important information on the hypotheses or algorithms into consideration \cite{Gu09}. Thus, it is fundamental to check in advance the adequateness of a performance score for assessing a determined classification problem so that the validity of the obtained results can be ensured. In this paper, we check this issue for the case of the class-imbalance domain: by studying the implication of the class distribution on the ``already known'' behaviour of the \textsf{BDR} as it is perceived by several numerical scores, which performance scores are inadequate for excluding important information about the behaviour of the classifier can be directly determined . 

Formally, we define a numerical performance score\footnote{Numerical performance scores are quantity measures which produce a single number to summarise the true confusion matrix. The graphical performance measures are left out of the scope of this paper.} as $\mathcal{S}_{\Psi}(\boldsymbol{\eta},\boldsymbol{\theta})$. Since it is a summary of the confusion matrix, it also depends on the same parameters $\boldsymbol{\eta}$, $\boldsymbol{\theta}$ and $\Psi=B$. Therefore, the numerical performance scores can be mainly divided into two different groups: local scores, which only focus on the behaviour of one target class, and global scores, which summarise the performance of the classifier taking into account its behaviour in all classes. The list of performance scores considered and their formal definitions can be found in Table \ref{scores}.

We use two well-known scores as {\it local performance scores}: precision $\mathcal{P}^{i}$ and recall $\mathcal{R}^{i}$. Whilst  $\mathcal{P}^{i}$ assesses to what extent the predictions of a certain class $c_i$ are correct, $\mathcal{R}^{i}$ assesses to what extent all examples of a certain class $c_i$ are classified as so. Unfortunately, for their local property, they lose the global picture of the performance of the classifier; they are more useful when combined with other scores or applied to all classes \cite{Gu09}. Therefore, most of the {\it global scores} are just functions which, by taking local performance scores applied to some/all classes, summarise the behaviour of the classifier according to a determined subjective criterion. In these studies, we have selected the global scores principally used or mentioned  in the class-imbalance literature, which can be expressed as H\"{o}lder means \cite{Bul03} among the recalls.
The classification accuracy, $\mathcal{A}cc$, is a weighted H\"{o}lder mean with parameters $p=1$ and $\zeta_i=\eta_i$.  As previously mentioned, the performance on underrepresented classes has very little impact on the measure when compared to the overrepresented classes \cite{Gu09}\cite{Mar13}. 
Due to this, most of the scores for unbalanced learning average over the recalls without weighting these values on the class probabilities. Therefore, all classes, over and underrepresented, share a common consideration in the score. The most-used scores in the class-imbalance literature are the Pythagorean means -- arithmetic ($\mathcal{A}$), geometric ($\mathcal{G}$), and harmonic ($\mathcal{H}$) means  -- over the recalls of the $K$-classes, which can be directly expressed as unweighted H\"{o}lder means with $p=1$, $p=0$ and $p=-1$, respectively. In the literature, they are referred to as a-mean, g-mean, and h-mean, respectively. Moreover, although they are not commonly used in the literature, we will also consider the extreme values of these unweighted H\"{o}lder means ($p=\infty$ and $p=-\infty$). They correspond to the maximum recall, $max$, and the minimum recall, $min$ (among the classes), respectively. We believe that they can also give valuable information in this complex scenario.

\begin{table}[t]
	\begin{center}
  		\small
    			\begin{tabular}{ |c| >{\centering}m{1.4in}|c| c@{}| }
    				\hline
				&Name	&Notation	&Formula \\
				\hline \hline
				\parbox[t]{1mm}{\multirow{2}{*}{\rotatebox[origin=c]{90}{\bf Local scores~\hspace{-\normalbaselineskip}}}} &Precision	&$\mathcal{P}^{i}$		&$\displaystyle a_{i,i}\left(\sum_{j=1}^Ka_{j,i}\right)^{-1}$ \\ \cline{2-4}
				&Recall &$\mathcal{R}^{i}$		&$\displaystyle  a_{i,i}\left(\sum_{j=1}^Ka_{i,j}\right)^{-1}$\\ \hline \hline
				\parbox[t]{1mm}{\multirow{5}{*}{\rotatebox[origin=c]{90}{\bf Global scores~\hspace{-\normalbaselineskip}}}} &Classification accuracy &$\mathcal{A}cc$	&$\displaystyle \sum_{i=1}^K \eta_i \mathcal{R}^{i}$ \\ \cline{2-4}
				&Arithmetic mean among the recalls ({\it a-mean}) &$\mathcal{A}$  &$\displaystyle \sum_{i=1}^K \frac{1}{K}\mathcal{R}^{i}$ \\ \cline{2-4}
				&Geometric mean among the recalls ({\it g-mean}) &$\mathcal{G}$  &$\displaystyle \sqrt[K]{\prod_{i=1}^K \mathcal{R}^{i}}$ \\ \cline{2-4}
				&Harmonic mean among the recalls ({\it h-mean}) &$\mathcal{H}$  &$\displaystyle K\left(\sum_{i=1}^K\frac{1}{\mathcal{R}^{i}}\right)^{-1}$ \\  \cline{2-4}
				&Maximum value among the recalls &$max$  &$\displaystyle \max_i \mathcal{R}^{i}$ \\  \cline{2-4}
				&Minimum value among the recalls &$min$  &$\displaystyle \min_i \mathcal{R}^{i}$ \\  
				 \hline
   			 \end{tabular}
        		\caption{Numerical performance scores (by convention  $0/0=1$).}
		\label{scores}
	\end{center}
\end{table}

In \cite{Men13}, it is stated that the \textsf{BDR} is not an optimal decision rule for $\mathcal{A}$, i.e. a higher value for the score may be obtained with other classifiers. Unfortunately, in the literature, no further information is provided either on which the optimal classifier for $\mathcal{A}$ is or on whether the \textsf{BDR} is optimal for the other two broadly used unweighted H\"{o}lder means ($\mathcal{G}$ and $\mathcal{H}$) or for the extreme means. In the following sections, we also shed light on these questions.

\section{First Study: Adequate Numerical Performance Scores for Unbalanced Problems}\label{sec:st1}

In this section, our goal is to answer {\it which performance scores are adequate to determine the competitiveness of a classifier in unbalanced domains?} Particularly, we want to be able to determine which performance scores succeed in expressing the long-studied performance detriment \cite{Jap02}\cite{Wei13} resulting from learning, in a classical manner, from skewed class distributions.

\subsection{A Novel Framework to Marginalise the Effect of the Class Distribution on the Performance Scores}

\begin{figure}[h] 
\centering
\begin{subfigure}{0.3\textwidth}
  \centering
  \includegraphics[width=.95\linewidth]{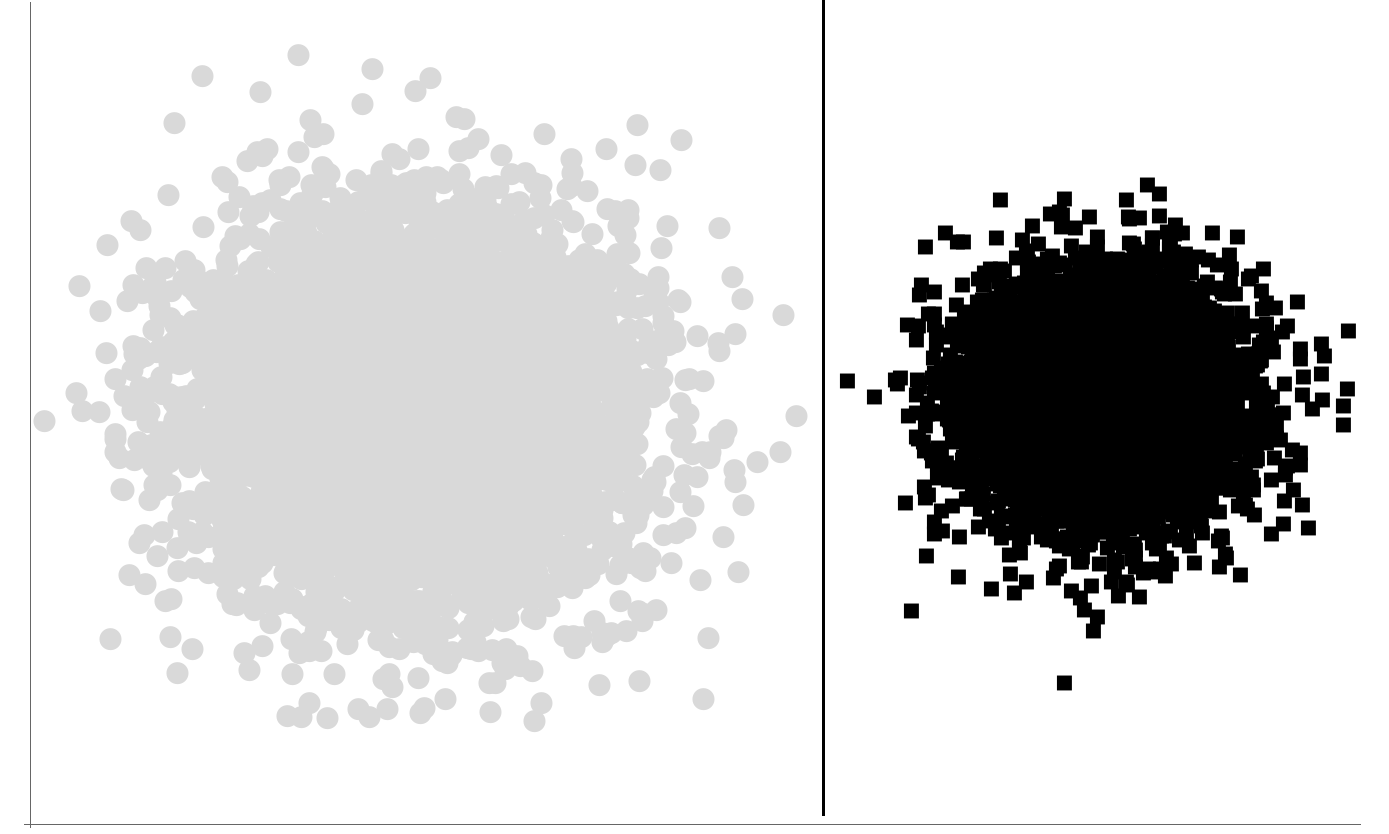}
  \caption{Balanced problem \\with no class-overlapping.}
  \label{subfig:comp1}
\end{subfigure}%
\begin{subfigure}{0.3\textwidth}
  \centering
  \includegraphics[width=.95\linewidth]{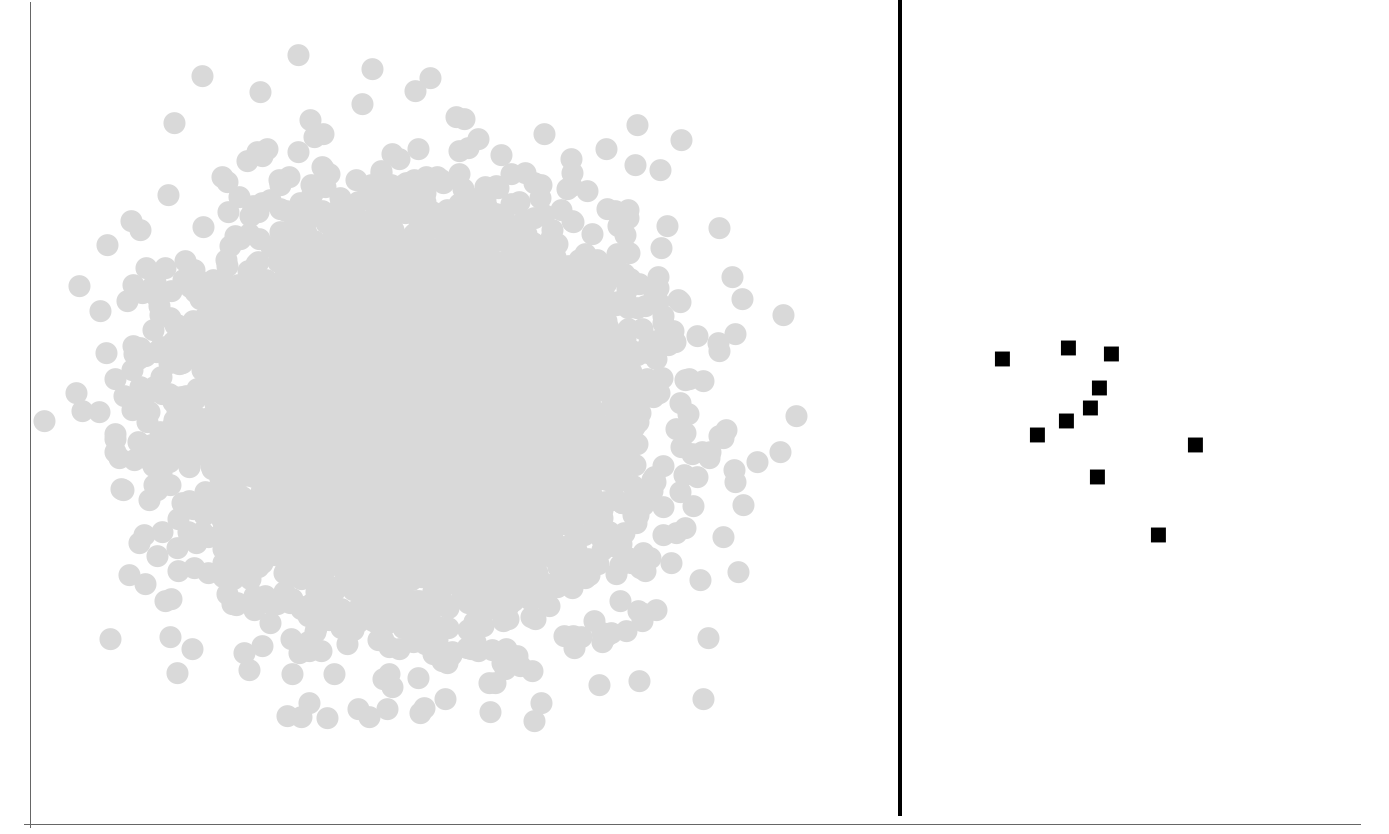}
  \caption{Unbalanced problem \\with no class-overlapping.}
  \label{subfig:comp2}
\end{subfigure}\\
\begin{subfigure}{0.3\textwidth}
  \centering
  \includegraphics[width=.95\linewidth]{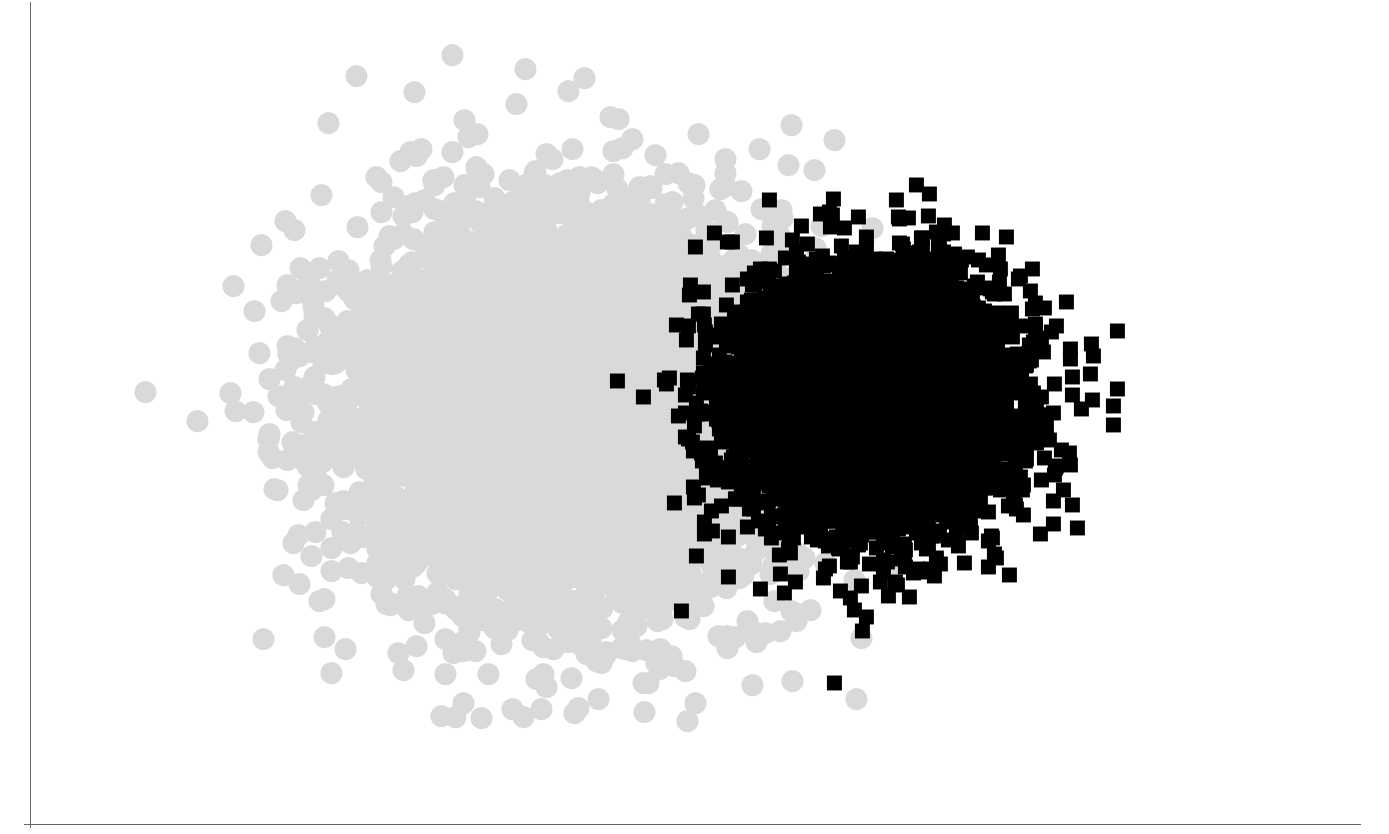}
  \caption{Balanced problem \\with class-overlapping.}
  \label{subfig:comp3}
\end{subfigure}%
\begin{subfigure}{0.3\textwidth}
  \centering
  \includegraphics[width=.95\linewidth]{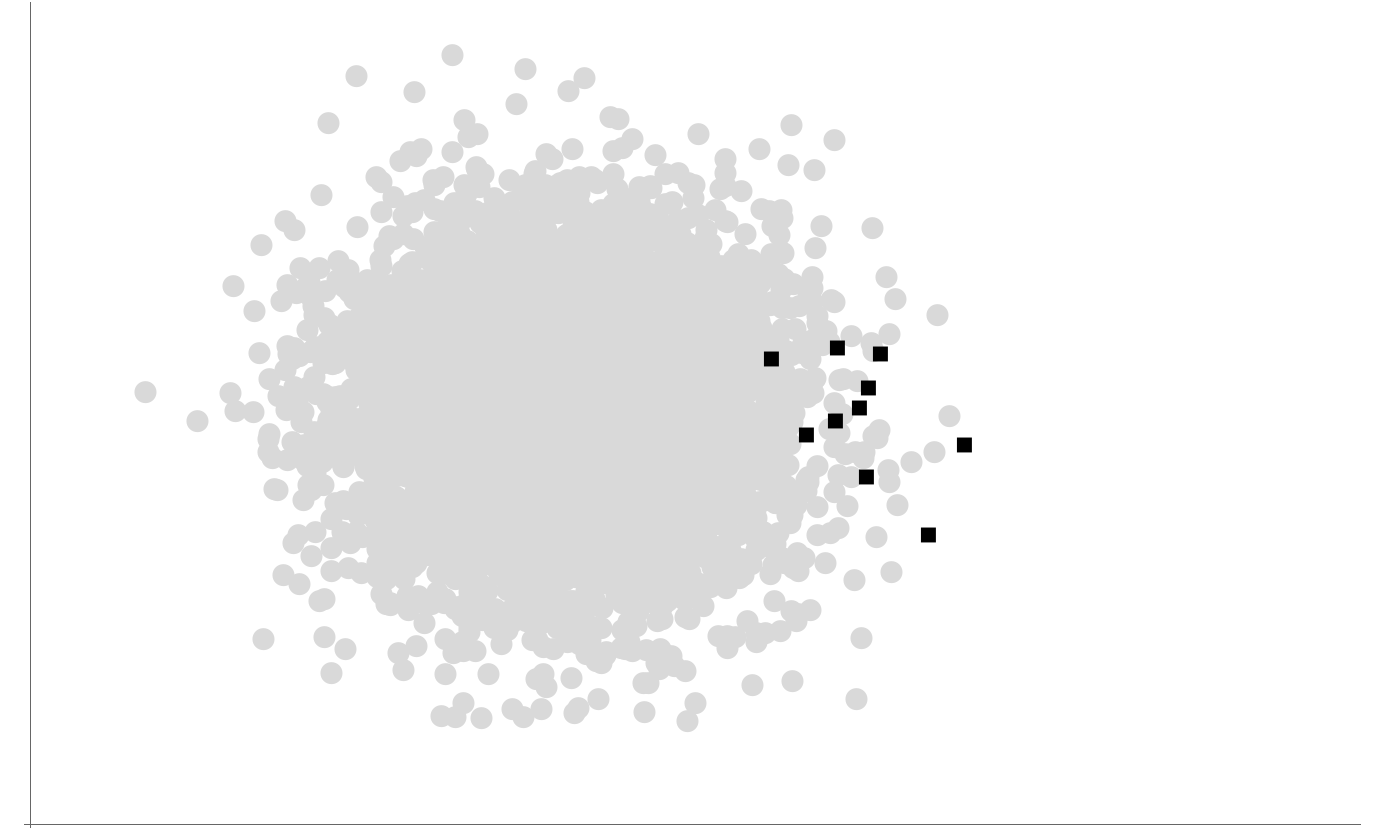}
  \caption{Unbalanced problem \\with class-overlapping.}
  \label{subfig:comp4}
\end{subfigure}
\caption{Relation between class-overlapping and class-imbalance.}
\label{fig:complexity}
\end{figure}

Recently, some authors \cite{Lop13}\cite{Den10} have stated that the class distribution is not the only factor hindering the predictive power of the classifiers. These other factors are listed in \cite{Lop13} as the degree of overlapping among the classes, the training size, the noise in the data, the presence of small disjuncts and the dispersity of some classes, among others. Moreover, these authors also argue that (i) the hindering factors have strong interdependences among them, and that (ii) these interdependences can modify the contribution of the class distribution to the behaviour of the classifier. Thus, it seems that it is not trivial to isolate the implication of the class distribution on the performance of the classifier so that, therefore, the adequateness of the numerical performance scores to capture this implication can be determined. Special care must be taken to marginalise out all the rest of the hindering factors. Fortunately, all of these causes, with the exception of class-overlapping, are only dependent on the nature of the training dataset. Since the \textsf{BDR} only depends on the generative function, this classifier allows us to omit practically all these factors which may modify the real impact that the imbalance extent has on the performance of the classifier. However, the effect of the class-overlapping is harder to eliminate; it depends on the local probability distributions of the feature space \cite{Den10}, i.e. class-overlapping, like class-imbalance, is an intrinsic characteristic of the generative model. Hence, the overlapping-imbalance dependence must be exhaustively studied in order to find a legitimate manner to remove the class-overlapping from this puzzling situation.

\subsubsection{Class-overlapping and Class-imbalance Relationship}

First, we set up an example in order to clarify the aforementioned dependency and how both the class-overlapping and class-imbalance factors may hinder the behaviour of the inferred classifier. Figure \ref{fig:complexity} shows four binary problems sharing two different degrees of class-imbalance and two different degrees of class-overlapping. For each problem, the data has been created by sampling two bivariate Gaussian distributions, one distribution for each class, i.e. $\boldsymbol{\theta}_i=\{\boldsymbol{\mu}_i,\boldsymbol{\Sigma}_i\}, i=\{1,2\}$. To simulate no class-overlapping, we choose two Gaussians whose means are far from each other (Figures \ref{subfig:comp1} and \ref{subfig:comp2}), and to simulate the opposite, we shorten the distance between these means (Figures \ref{subfig:comp3} and \ref{subfig:comp4}). Regarding the class-imbalance setting, it is simulated by sampling $1,000$ instances of each class for balanced problems (Figures \ref{subfig:comp1} and \ref{subfig:comp3}), and by sampling $1,000$ instances for the majority class and only $10$ samples for the minority, for the case of unbalanced problems (Figures \ref{subfig:comp2} and \ref{subfig:comp4}). By just having an overall look at Figure \ref{fig:complexity}, it can be easily noticed how the combination of both factors (class-imbalance and class-overlapping) has a straight effect on the issue of discriminating among the classes and, therefore, on the performance of the inferred classifier. When there is no class-overlapping, as shown in the first row of the figures, the class distribution does not hinder the predictive power of the resulting classifier. In both scenarios, a simple and perfect discriminant linear classifier can be easily drawn. This classifier is represented by a straight continuous line in the figure. However, in the second row, the situation is completely different. When both classes are balanced (Figure \ref{subfig:comp3}), a classifier with a tolerable recall for both classes can be learned. Unfortunately, in the unbalanced case (Figure \ref{subfig:comp4}), the lack of enough examples for the minority class hinders the process of discriminating that class; any intuitively chosen classifier will be incompetent, i.e. it will have a low recall for the minority class. This example concurs with the literature claims \cite{Pra04}, where it is stated that (i) only when the class-overlapping is non-zero, the influence of the class distribution on the competitiveness of the inferred classifier is noticeable, and that (ii) the influence of class-overlapping into the learning process is even stronger than class-imbalance. 

\subsubsection{Isolating Class-imbalance from Class-overlapping}

Whilst class-overlapping is an old stalwart in the literature for being broadly studied \cite{Bas06}, to the best of our knowledge, no prior work in the literature has isolated the impact that the class distribution has on the performance of the classifiers or has explored which performance scores are able to appropriately capture this potential fluctuation of performance. In order to bridge these gaps, we propose a function which, by properly cancelling the effect of the class-overlapping, returns the impact of the class distribution on a chosen numerical performance score for assessing the inferred classifier (here \textsf{BDR}): 

\begin{definition} Let $\boldsymbol{\Theta}$ represent the space of parameters for a fixed family of distributions over the feature space for classification problems with $K$ classes. Let $\mathcal{S}_{B}(\boldsymbol{\eta},\boldsymbol{\theta})$ be the value of a performance score $\mathcal{S}$ assessing the behaviour of the \textsf{BDR} on a $K$-class classification problem, $\gamma_K$, with a class distribution equal to $\boldsymbol{\eta}$ and parameters $\boldsymbol{\theta} \in \boldsymbol{\Theta}$. Also, let $\mathcal{S}_{B}(\boldsymbol{e},\boldsymbol{\theta})$ be the value of $S$ evaluating the \textsf{BDR} inferred from the balanced version of $\gamma_K$. Therefore, the {\it influence function}, $\mathds{I}^{\mathcal{S}}_K(\boldsymbol{\eta})$, of the $K$-class distribution $\boldsymbol{\eta}$ on the performance score\footnote{In this paper, we assume a positive correlation between the value of the performance score and the behaviour of the classifier, i.e. higher values of $\mathcal{S}$ represent higher performances. In the event of a negative correlation, the sign of $\mathds{I}^{\mathcal{S}}_K(\boldsymbol{\eta})$ must be reversed.} $\mathcal{S}$ using the \textsf{BDR} as a classifier is defined as follows:
\begin{equation}
\mathds{I}^{\mathcal{S}}_K(\boldsymbol{\eta})= \displaystyle\int\limits_{\boldsymbol{\Theta}}{}[\mathcal{S}_{B}(\boldsymbol{e},\boldsymbol{\theta})-\mathcal{S}_{B}(\boldsymbol{\eta},\boldsymbol{\theta})]d\boldsymbol{\theta}.
\label{imbfeat}
\end{equation}
\end{definition}

\noindent It can be easily seen that the previous equation fulfils our couple of objectives. First, the implication of the class-overlapping on the behaviour of the inferred classifier is taken out of the equation by means of the integration; every possible set of parameter values for the probability distributions of the generative model showing a non-zero degree of overlapping among the classes is marginalised out. As a result, the average influence of the class distribution on the behaviour of the inferred classifier, as conceived by the performance score $\mathcal{S}$, can then be quantified and studied: assuming a fixed parametric family for the local distribution, positive values of $\mathds{I}^{\mathcal{S}}_K(\boldsymbol{\eta})$ denote that the performance score $\mathcal{S}$ obtains, on average, higher values for the \textsf{BDR} when it is used on a balanced scenario rather than when it is inferred from a class distribution $\boldsymbol{\eta}$. Negative values mean the opposite; the \textsf{BDR} achieves, in general, worse values for $\mathcal{S}$  in the balanced scenario.

\subsection{Identifying the Adequate Performance Scores}

By just plotting $\mathds{I}^{\mathcal{S}}_K(\boldsymbol{\eta}), \forall \boldsymbol{\eta}$ using different performance scores, we can perceive, at a glance, how these scores differ in measuring the goodness of the traditional classifiers. However, we are incapable of determining which of them are adequate to validate classifiers in unbalanced domains. Thus, in order to accomplish the latter task, we must take advantage of the prior studies \cite{Gu09}\cite{Dru05}\cite{Axe00} on the expected behaviour of the popular \textsf{BDR} when it faces skewed class distributions. By doing so, we can determine the shape that an adequate performance score should have for the influence function in the class-imbalance spectrum. Then, this shape can be straightforwardly used as a representative case to discern which performance scores are appropriate for the class-imbalance scenario. For this purpose, we focus on the long discussed hindering behaviours of the \textsf{BDR}:
\begin{enumerate}[label=(\roman*)]
\item A good prediction power of the \textsf{BDR} is only guaranteed for the majority classes \cite{Gu09}\cite{Axe00}. Therefore, the best performance of the \textsf{BDR} for all classes should occur when they share the same class probability, i.e. the balanced scenario.
\item In highly unbalanced situations, the \textsf{BDR} often performs little better than a dummy classifier always predicting the most common classes \cite{Dru05}. Hence, the performance of the \textsf{BDR} should be highly penalised in those situations.
\end{enumerate}
Then, from these hindering aspects, we define the shape that the influence function of an adequate performance score should have:

\noindent{\bf Properties of an adequate performance score.} {\it A performance score $\mathcal{S}$ is successfull in being adequate to determine the competitiveness of a classifier $\Psi$ in the class-imbalance scenario if, assuming a scenario where a classifier is inferred by directly minimising a $0$-$1$ loss (maximising the classification accuracy),}
\begin{enumerate}[label=(\alph*)]
\item {\it its influence function is positive for almost any $\boldsymbol{\eta} \neq \mathbf{e}$, and}
\item {\it it shows a negative correlation to the minority class probability, i.e. $\mathds{I}^{\mathcal{S}}_K(\boldsymbol{\eta})$ grows as the Euclidean distance between $\boldsymbol{\eta}$ and $\mathbf{e}$ gets larger.}
\end{enumerate}

\subsubsection{Experimental Model for the Study}

When the generative model is known, in theory, eq. (\ref{imbfeat}) is obtainable. However, solving an integral of such characteristics with independence of the parametric family is intricate. Therefore, we assume a parametric family with the following characteristics: (i) simple enough to be able to fully interpret the results and complete enough to be able to represent real world problems, (ii) a set of parameters which allows us to unambiguously represent each particular model as a single point. For these reasons, as a generative model we make use of a {\it univariate Gaussian identifiable mixture of components with unit variances whose means are separated by a fixed overlapping factor $\delta$}. Under this assumption, since each $\boldsymbol{\theta}_i=\{\mu_i,\sigma_i\}$ is such that $\mu_i=(i-1)\delta$ and $\sigma_i=1$, the parameters can be simplified to just $\boldsymbol{\theta}=\{\delta\}$ and eq. (\ref{imbfeat}) be rewritten as:
\begin{equation}
\mathds{I}^{\mathcal{S}}_K(\boldsymbol{\eta})= \displaystyle\int\limits_{0}^{\infty}[\mathcal{S}_{B}(\boldsymbol{e},\delta)-\mathcal{S}_{B}(\boldsymbol{\eta},\delta)]d\delta.
\label{imbfeat2}
\end{equation}
In Figure \ref{fig:binary1} and \ref{fig:binary2} (local and global scores for binary problems, respectively) and Figure \ref{fig:multi2} (global scores in the multi-class framework) we numerically approximate\footnote{{\bf NIntegrate} with all options set to default.}, using Mathematica \cite{Mathematica10}, the influence function of eq. (\ref{imbfeat2}), for each performance score of Table \ref{scores} and through the whole class-imbalance spectrum. For {\it binary problems}, we numerically approximate the following equation:
\begin{equation}
\mathds{I}^{\mathcal{S}}_2(\eta)= \displaystyle\int\limits_{0.01}^{10}[\mathcal{S}_{B}(\frac{1}{2},\delta)-\mathcal{S}_{B}(\eta,\delta)]d\delta.
\label{binarycomp}
\end{equation}
\noindent Here, since, in binary problems, there are only two class probabilities which are complementary to each other, we can simplify $\boldsymbol{\eta}$ to $\eta$ ($\eta_1=\eta$ and $\eta_2=1-\eta$). Also, note that the integral is calculated in the domain $[0.01,10]$, instead of $[0,\infty)$. We choose an upper limit distance of $10$ because, for unit variances, it is almost equivalent to not overlapping. Regarding the lower limit, we choose $0.01$ in order to avoid the singularity $\delta=0$\footnote{Assuming an upper limit of $10$ instead of infinity, the degree of overlapping ($e_B$) is $2.9 \times 10^{-7}$ rather than $0$. Analogously, assuming a lower limit of $0.01$ means that, instead of having an overlapping of $0.5$, we have $0.498$.}. 

Regarding the {\it multi-class framework}, we perform the same case study. However, as this setting is more complex, several changes are made. First, for the sake of clarity in the presentation of the results, we simplify the multi-class framework to the following: $(K-1)$ classes are equiprobable among them, with probability $\eta=\frac{1}{K} - \frac{\epsilon}{K-1}$, whilst the remainder has a probability $\eta_1=\frac{1}{K}+ \epsilon$. Then, by means of just one parameter $\epsilon \in [-\frac{1}{K},\frac{K-1}{K}]$ which determines the imbalance extent, we can easily study the hindrance produced by the class distribution in the multi-class framework and present the results in a bi-dimensional plot. Note that the value $\epsilon=0$ corresponds to the balanced setting. Therefore, we numerically approximate the following equation:
\begin{equation}
\mathds{I}^{\mathcal{S}}_K(\epsilon)= \displaystyle\int\limits_{0.01}^{10}[\mathcal{S}_{B}(0,\delta)-\mathcal{S}_{B}(\epsilon,\delta)]d\delta.
\label{multiclasscomp}
\end{equation}
\noindent Here, all the parameters\footnote{Assuming an upper limit of $10$, the degree of overlapping for $K=\{3,4,5\}$ is $\{3.8,4.3,4.6\}\times 10^{-7}$ instead of $0$. Regarding the lower limit of $0.01$, the overlapping is $\{0.64,0.72,0.77\}$ instead of $(K-1)/K$.} are the same as in binary problems except for the defined $\epsilon$. Due to the limiting space and since similar results are obtained for any arbitrary $K$, in this manuscript, only the case of $K=3$ is presented. The source code to calculate $\mathds{I}^{\mathcal{S}}_K(\epsilon)$ for any number $K$ of classes can be downloaded from {\footnotesize\url{http://github.com/jonathanSS/ClassImbalanceStudies}}. 

\subsection{Results and Discussion}

\subsubsection{Binary Problems}

Figure \ref{fig:binary1} shows the value of the function $\mathds{I}^{\mathcal{S}}_2(\eta)$ over the domain $0 \leq \eta \leq 1$ for the local performance scores of Table \ref{scores} when using the \textsf{BDR}. Specifically, the precision, $\mathcal{P}^{1}$ is shown in Figure \ref{fig:binaryprecision} and the recall, $\mathcal{R}^{1}$, in Figure \ref{fig:binaryrecall}, both for class $c_1$. Note that, for local scores, the diagrams for $c_2$ are omitted. This is due to the fact that they are a reflection with respect to the imaginary vertical axis $\eta=0.5$ of the ones of $c_1$. Next, the values of the influence function for the global performance scores for binary problems are presented in Figure \ref{fig:binary2}. There, 
the accuracy, $\mathcal{A}cc$ (Figure \ref{fig:binaryacc}), 
the maximum recall, $max$ (Figure \ref{fig:binarymax}),
the arithmetic mean, $\mathcal{A}$ (Figure \ref{fig:binaryam}), 
the geometric mean, $\mathcal{G}$ (Figure \ref{fig:binarygm}), 
the harmonic mean, $\mathcal{H}$ (Figure \ref{fig:binaryhm}),
and the minimum recall, $min$ (Figure \ref{fig:binarymin}) 
are displayed. All plots share the same style; the $x$-axis represents the value of $\eta$ and the $y$-axis, $\mathds{I}^{\mathcal{S}}_2(\eta)$. The area between the function and the $x$-axis is highlighted for visual purposes.

\begin{figure*}[!th] 
\centering
\begin{subfigure}{.5\textwidth}
  \centering
  \includegraphics[width=.80\linewidth]{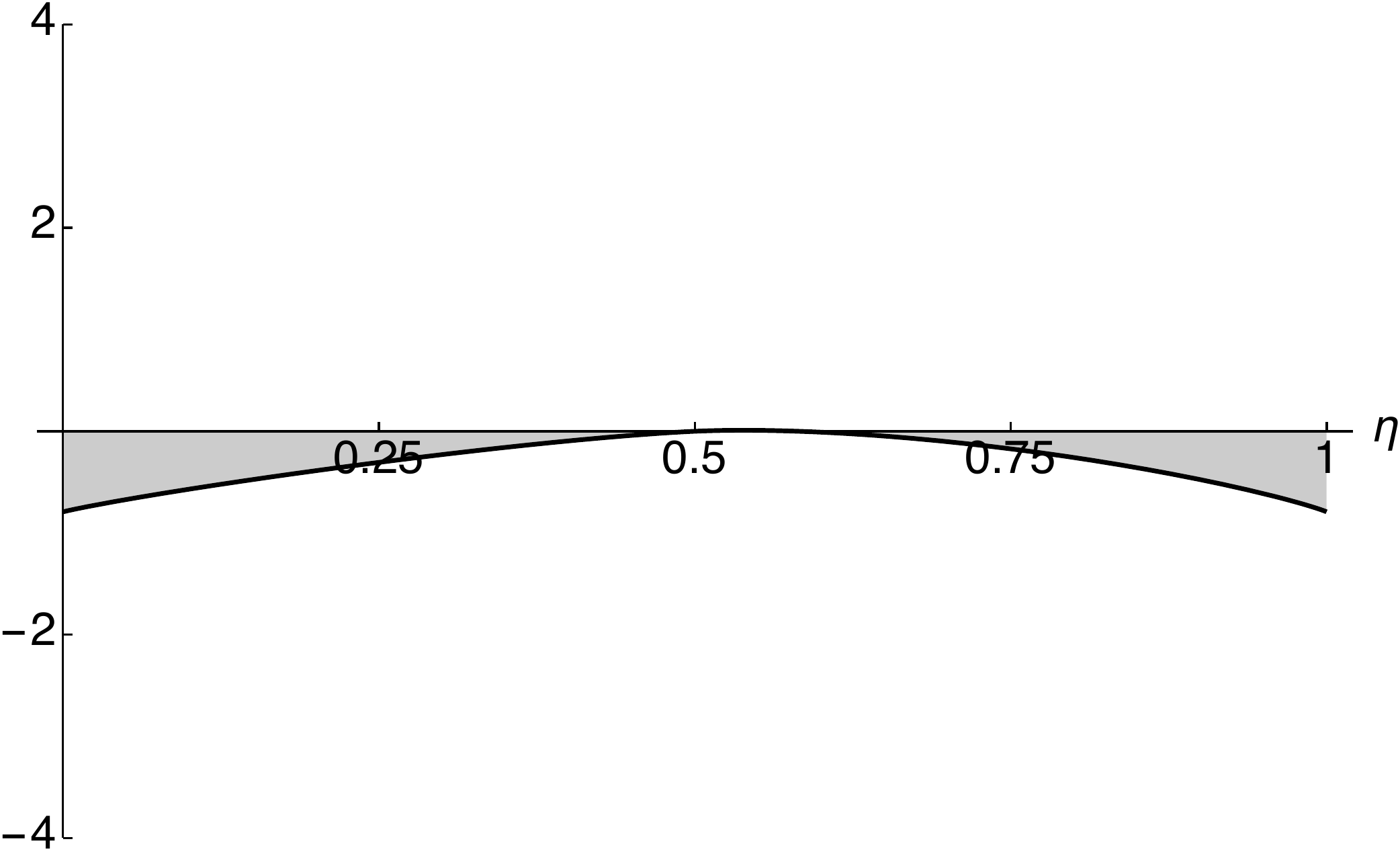}
  \caption{Precision ($\mathcal{P}^{1}$)}
  \label{fig:binaryprecision}
\end{subfigure}%
\begin{subfigure}{.5\textwidth}
  \centering
  \includegraphics[width=.80\linewidth]{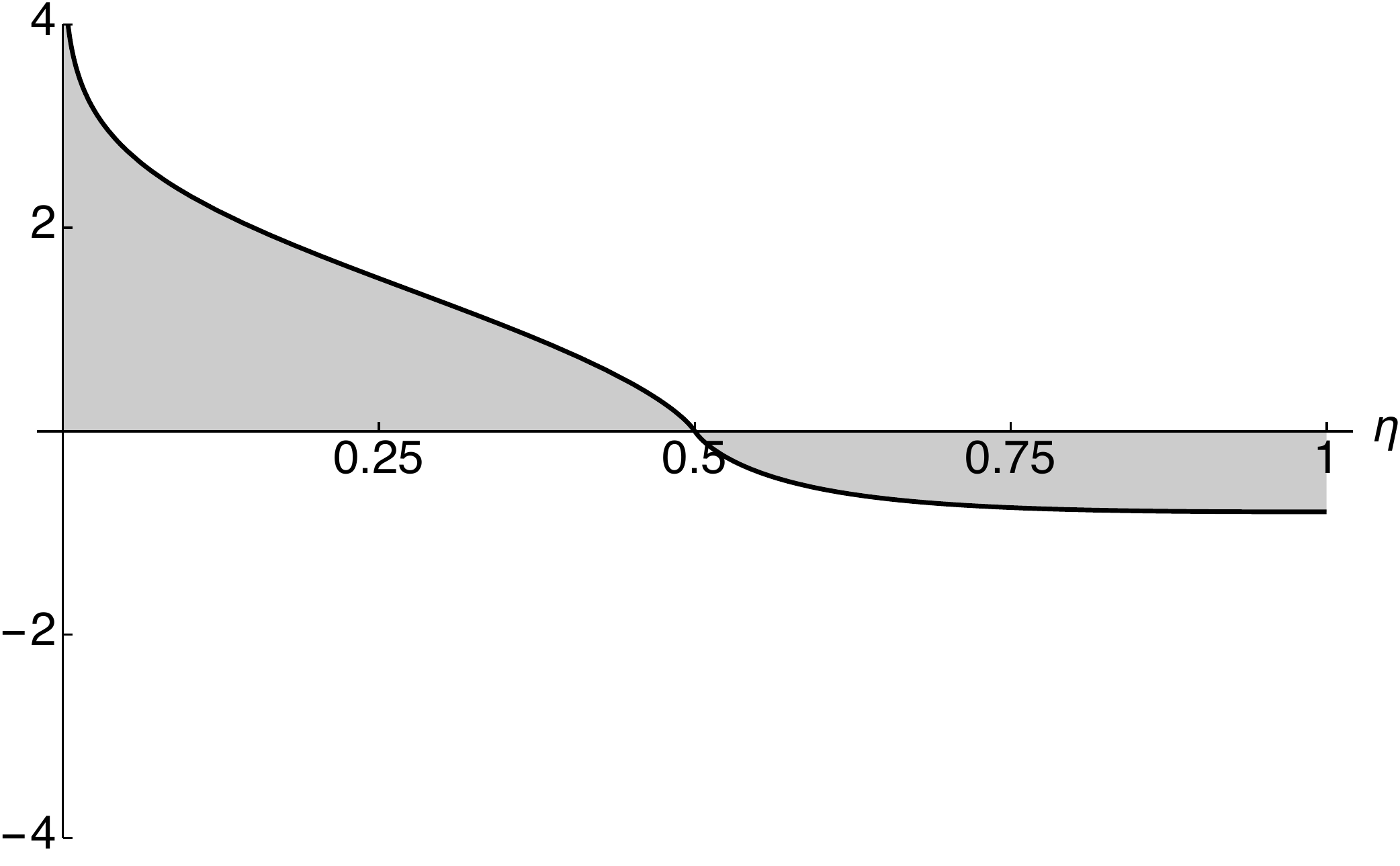}
  \caption{Recall ($\mathcal{R}^{1}$)}
  \label{fig:binaryrecall}
\end{subfigure}
\caption{The influence function in binary problems, $\mathds{I}^{\mathcal{S}}_2(\eta)$, for both precision and recall throughout the range $0 \leq \eta \leq 1$  ($c_1$).}
\label{fig:binary1}
\end{figure*}

\begin{figure*}[!h] 
\centering
\begin{subfigure}{.33\textwidth}
  \centering
  \includegraphics[width=\linewidth]{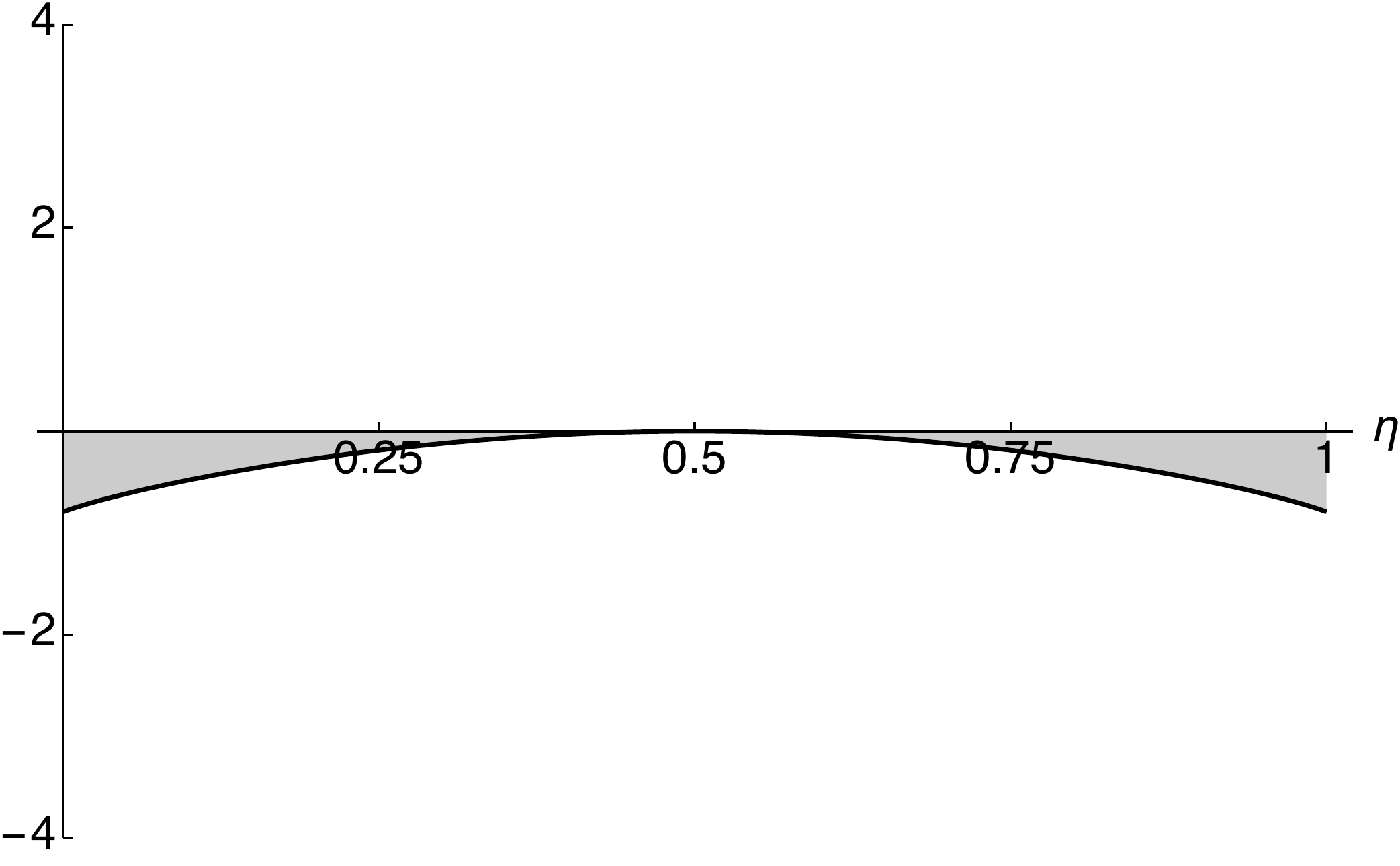}
  \caption{Classification accuracy ($\mathcal{A}cc$)}
  \label{fig:binaryacc}
\end{subfigure}%
\begin{subfigure}{.33\textwidth}
  \centering
  \includegraphics[width=\linewidth]{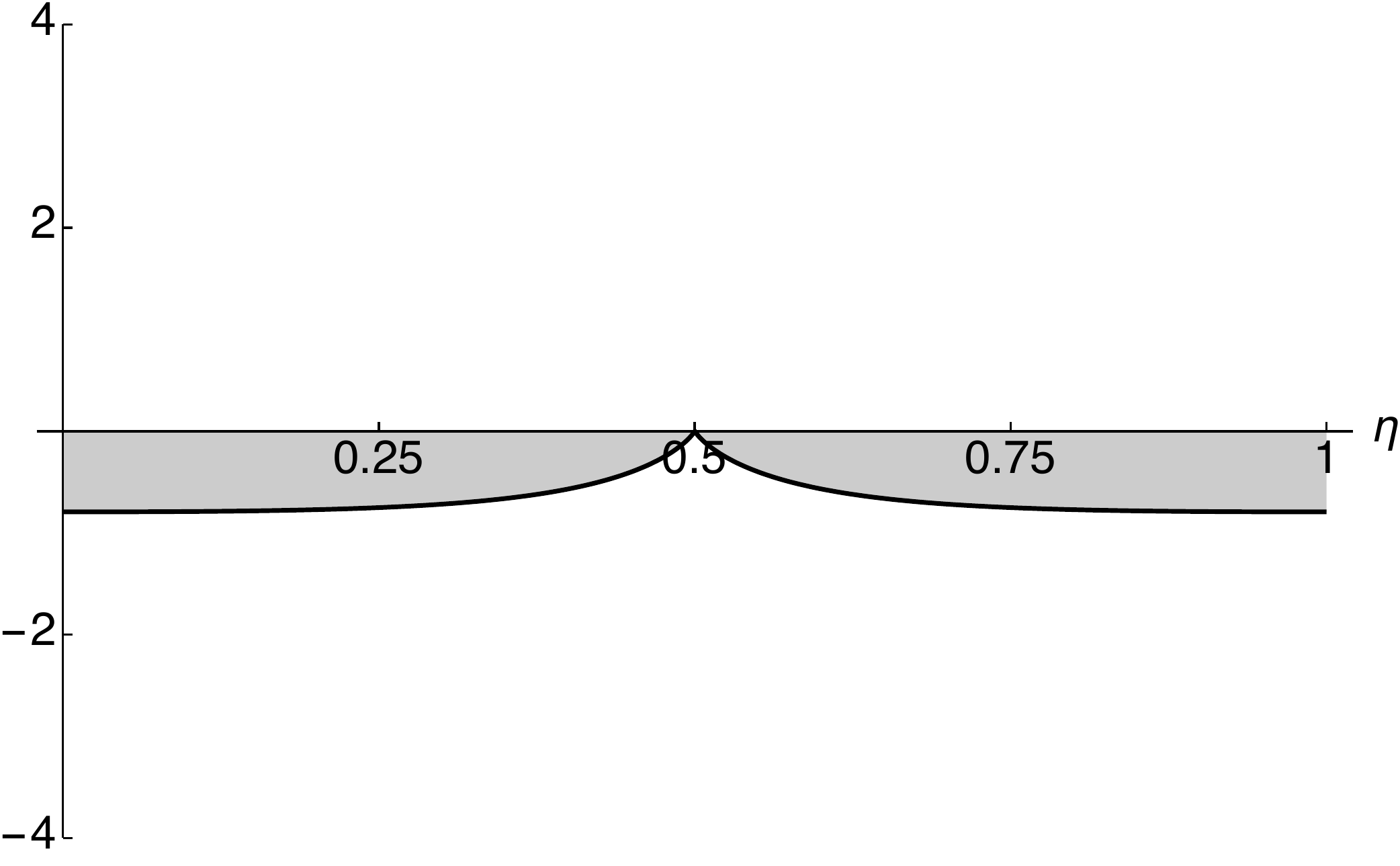}
  \caption{Maximum recall ($max$)}
  \label{fig:binarymax}
\end{subfigure}%
\begin{subfigure}{.33\textwidth}
  \centering
  \includegraphics[width=\linewidth]{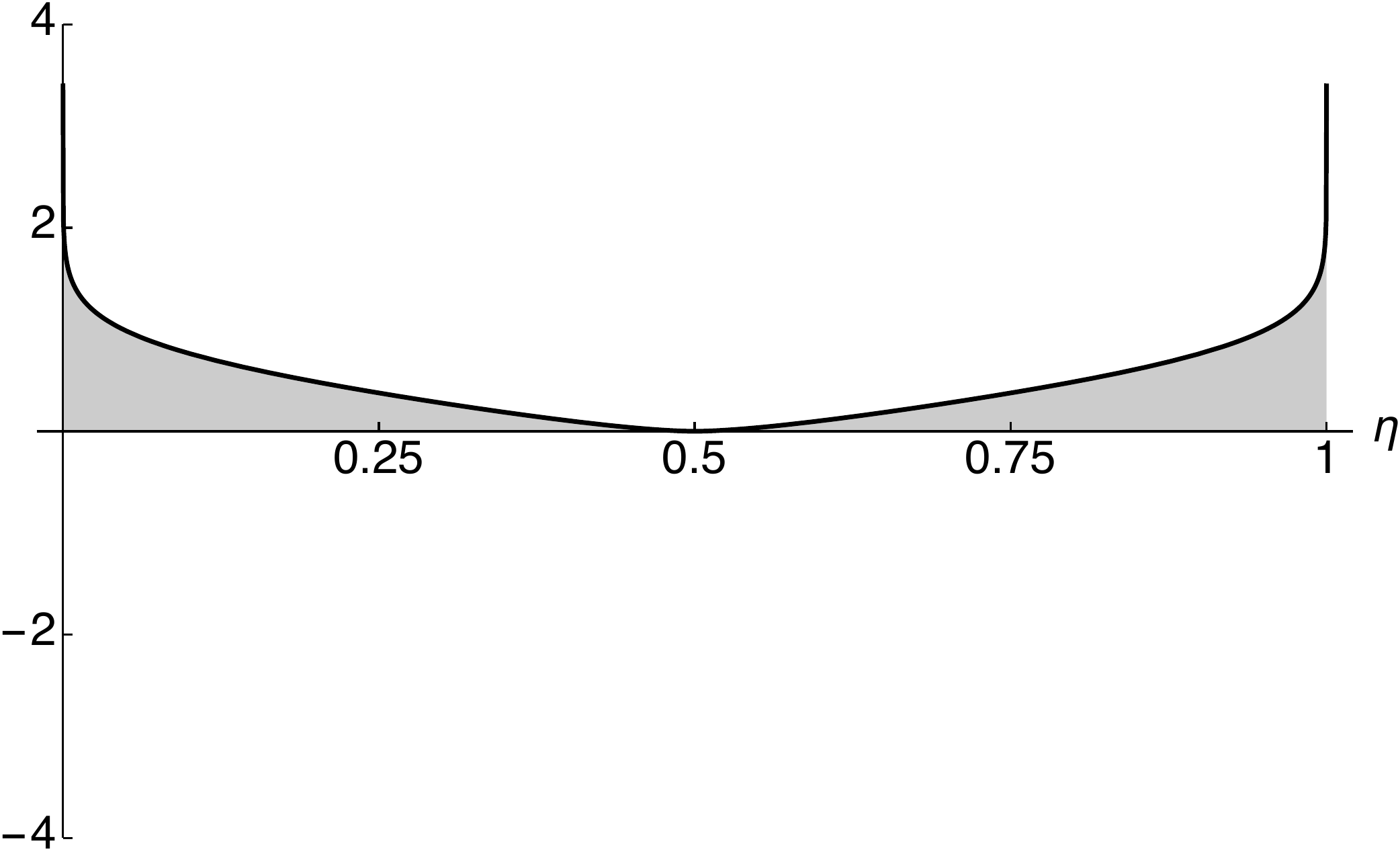}
  \caption{Arithmetic mean ($\mathcal{A}$)}
  \label{fig:binaryam}
\end{subfigure}
\begin{subfigure}{.33\textwidth}
  \centering
  \includegraphics[width=\linewidth]{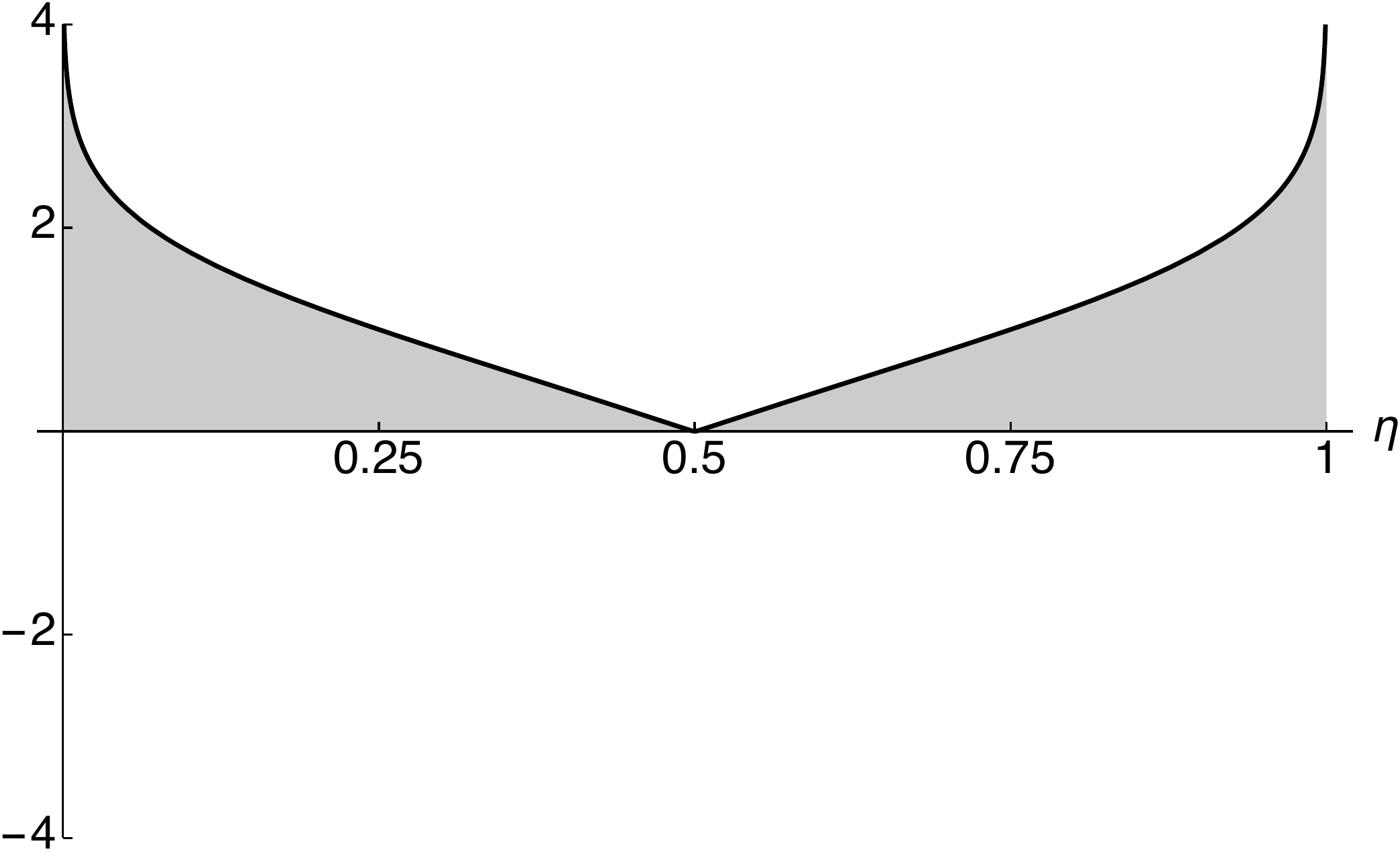}
  \caption{Geometric mean ($\mathcal{G}$)}
  \label{fig:binarygm}
\end{subfigure}%
\begin{subfigure}{.33\textwidth}
  \centering
  \includegraphics[width=\linewidth]{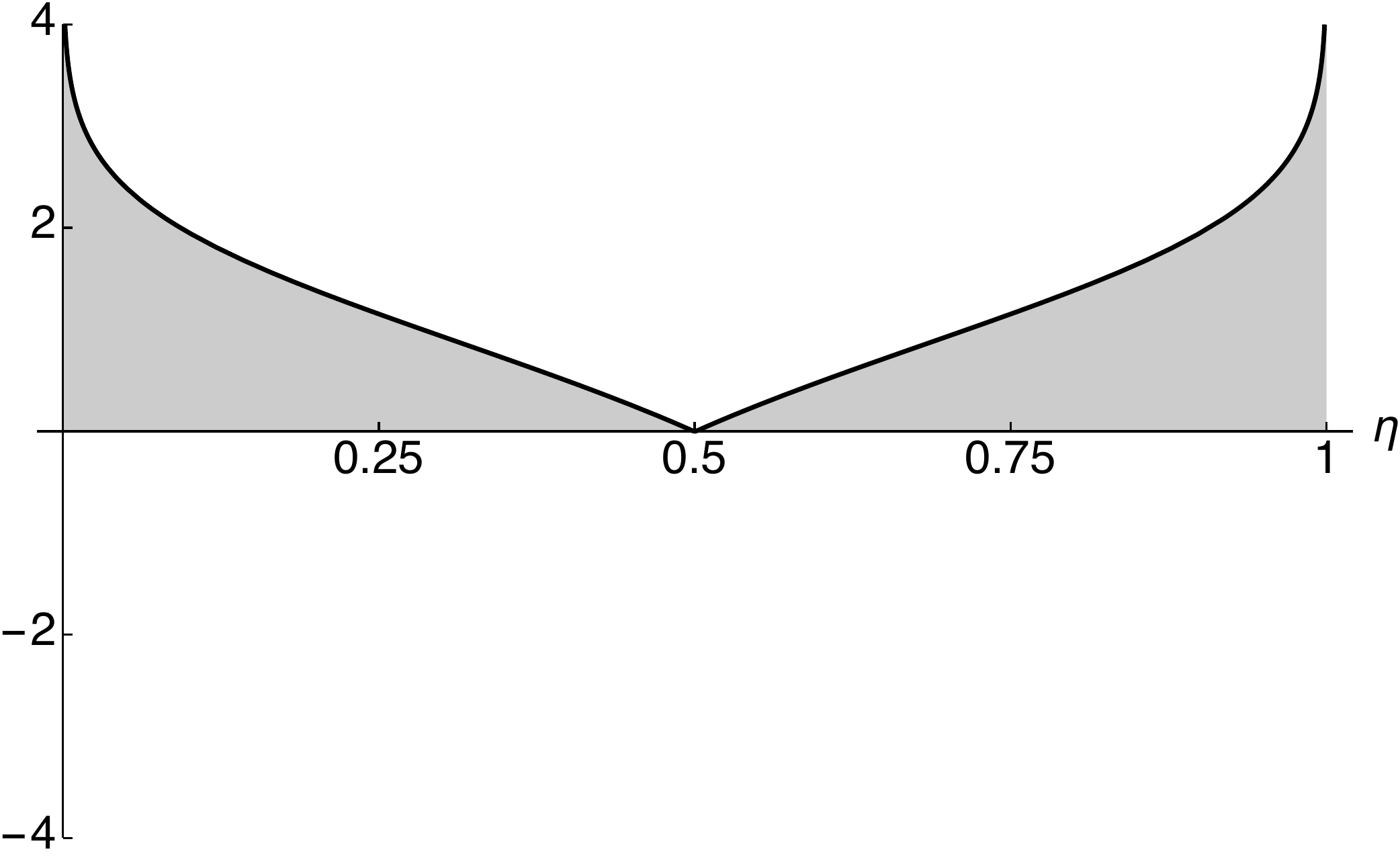}
  \caption{Harmonic mean ($\mathcal{H}$)}
  \label{fig:binaryhm}
\end{subfigure}%
\begin{subfigure}{.33\textwidth}
  \centering
  \includegraphics[width=\linewidth]{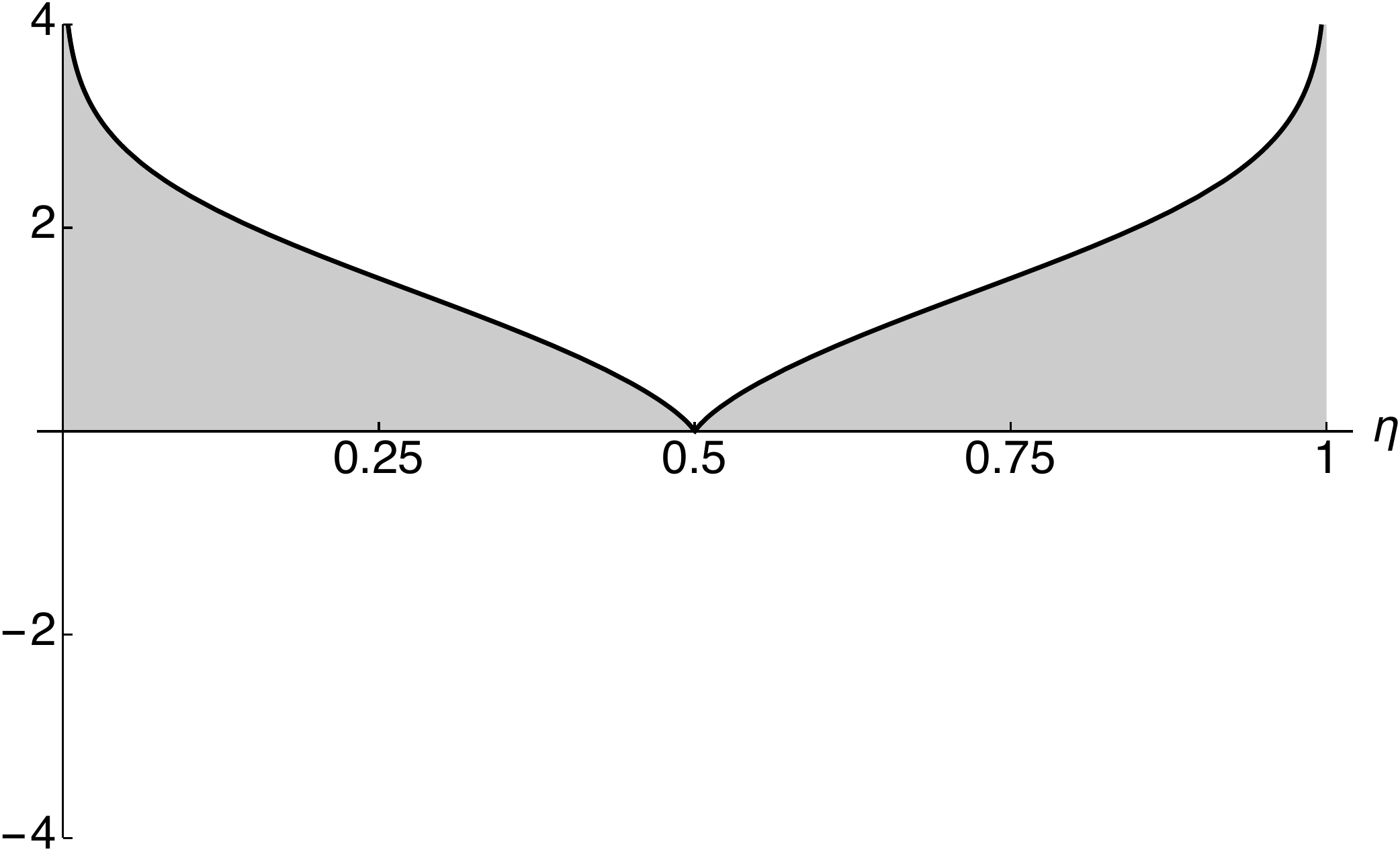}
  \caption{Minimum recall ($min$)}
  \label{fig:binarymin}
\end{subfigure}
\caption{The influence function in binary problems, $\mathds{I}^{\mathcal{S}}_2(\eta)$, for each global performance score throughout the range $0 \leq \eta \leq 1$.}
\label{fig:binary2}
\end{figure*}

\begin{figure*}[!h] 
\centering
\begin{subfigure}{.33\textwidth}
  \centering
  \includegraphics[width=\linewidth]{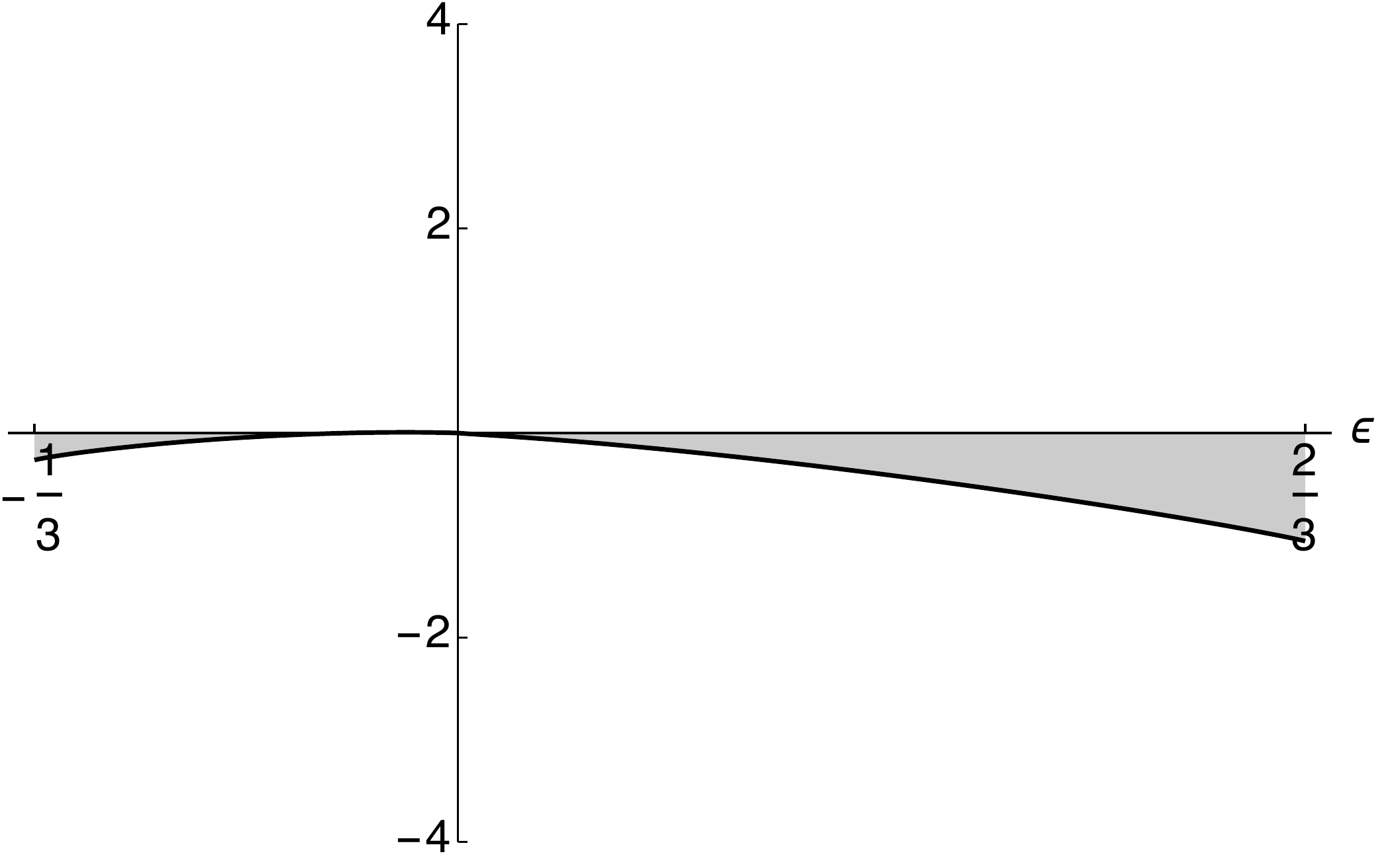}
  \caption{Classification accuracy ($\mathcal{A}cc$)}
  \label{fig:multiacc}
\end{subfigure}%
\begin{subfigure}{.33\textwidth}
  \centering
  \includegraphics[width=\linewidth]{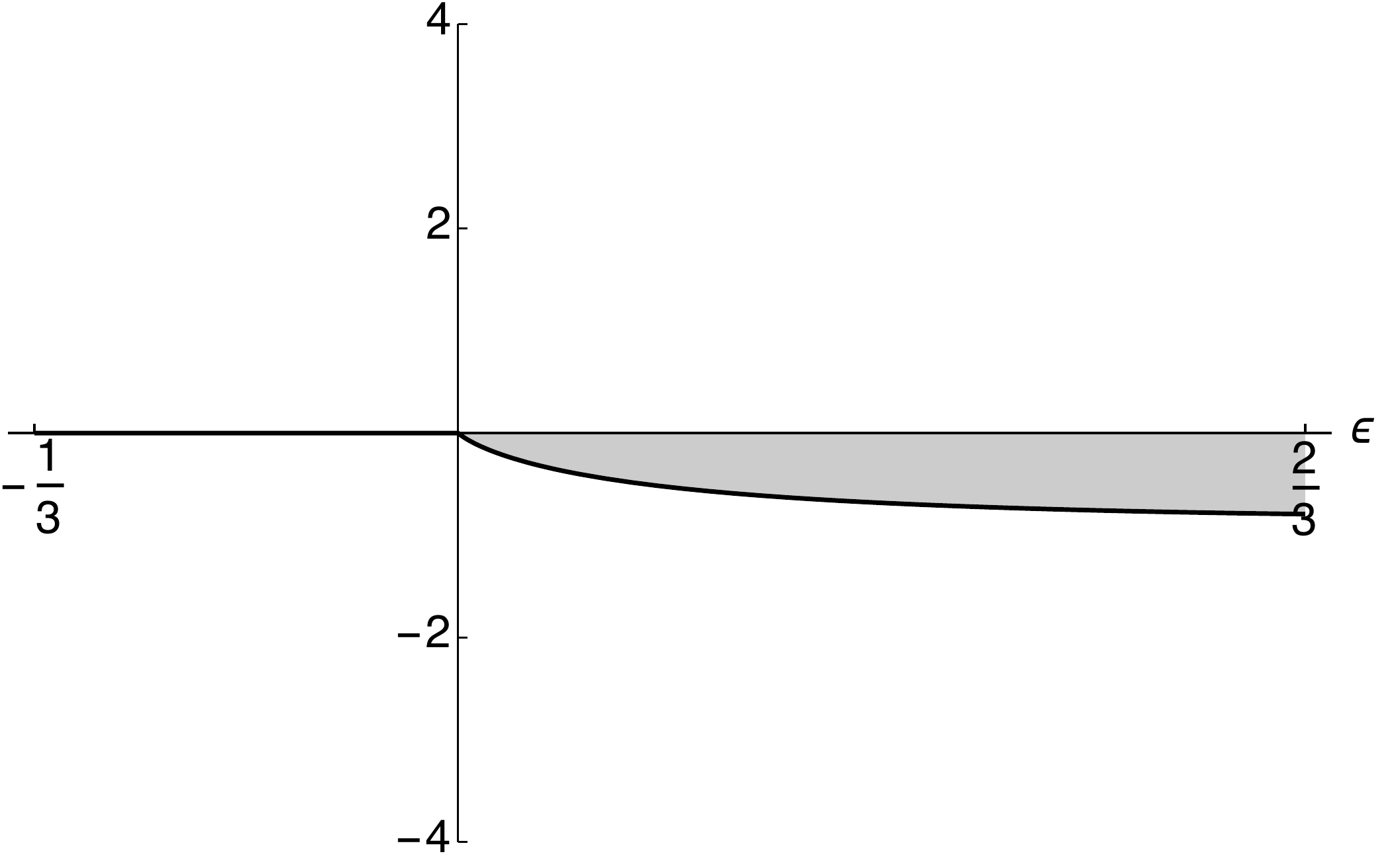}
  \caption{Maximum recall ($max$)}
  \label{fig:multimax}
\end{subfigure}%
\begin{subfigure}{.33\textwidth}
  \centering
  \includegraphics[width=\linewidth]{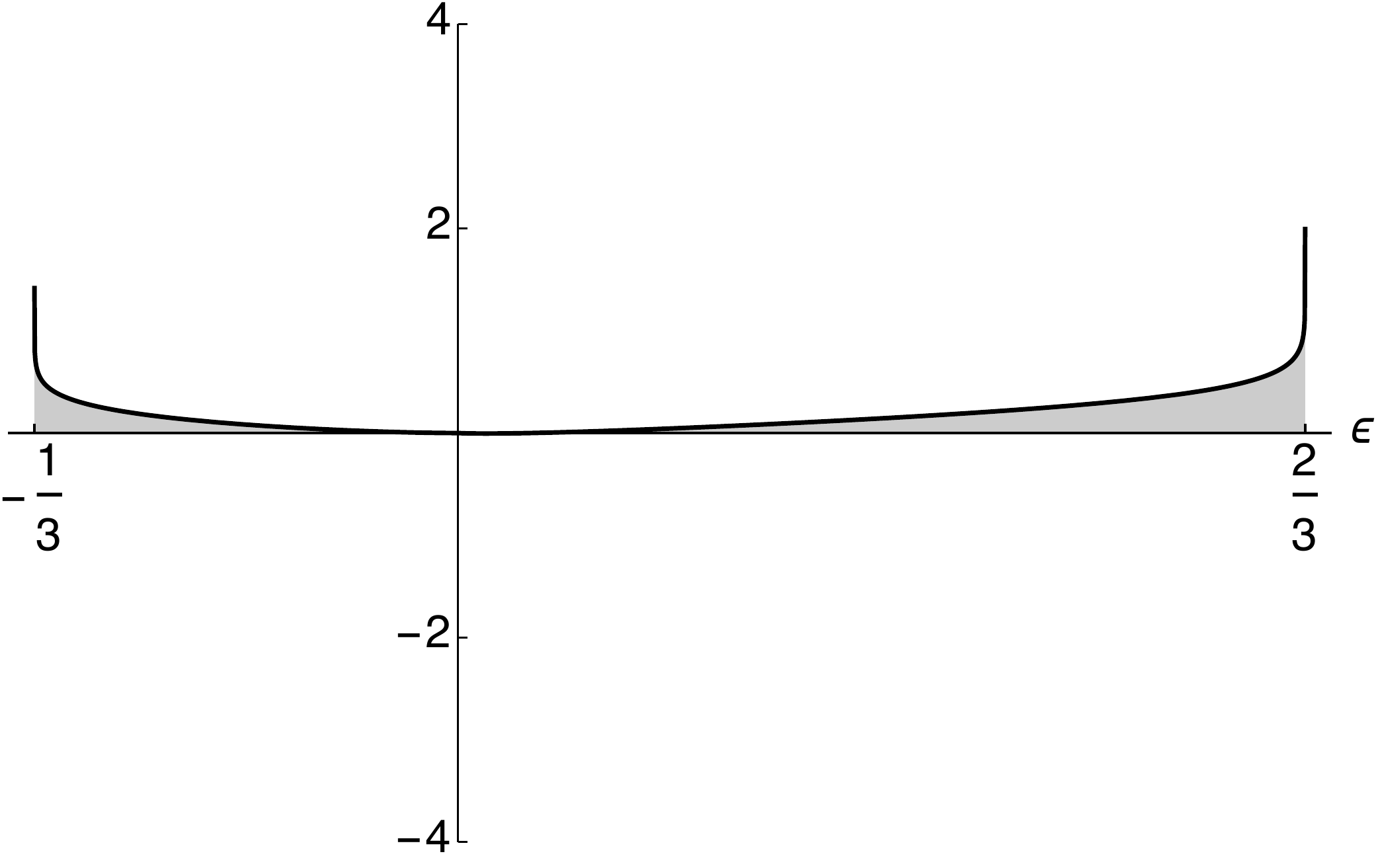}
  \caption{Arithmetic mean ($\mathcal{A}$)}
  \label{fig:multiam}
\end{subfigure}
\begin{subfigure}{.33\textwidth}
  \centering
  \includegraphics[width=\linewidth]{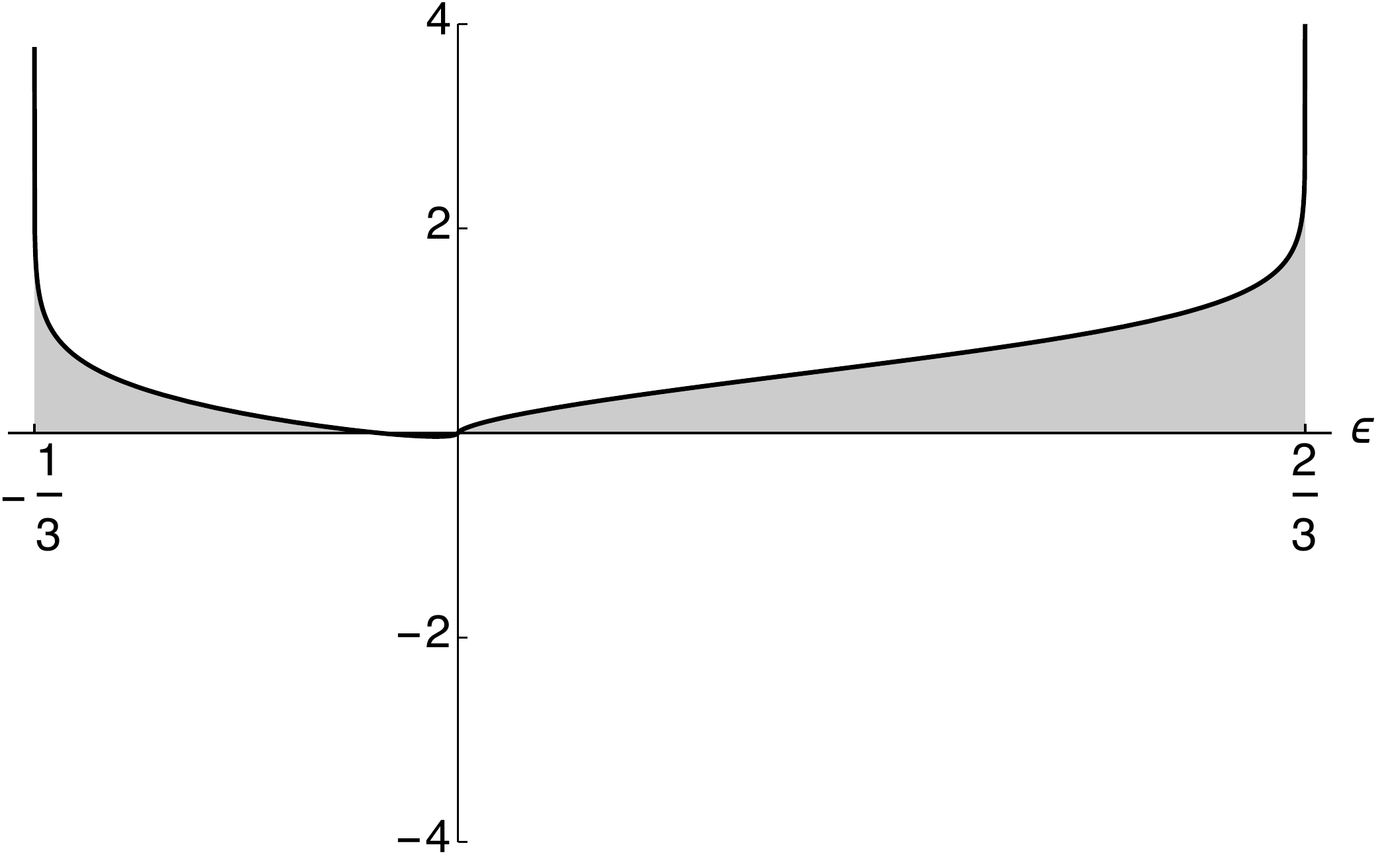}
  \caption{Geometric mean ($\mathcal{G}$)}
  \label{fig:multigm}
\end{subfigure}%
\begin{subfigure}{.33\textwidth}
  \centering
  \includegraphics[width=\linewidth]{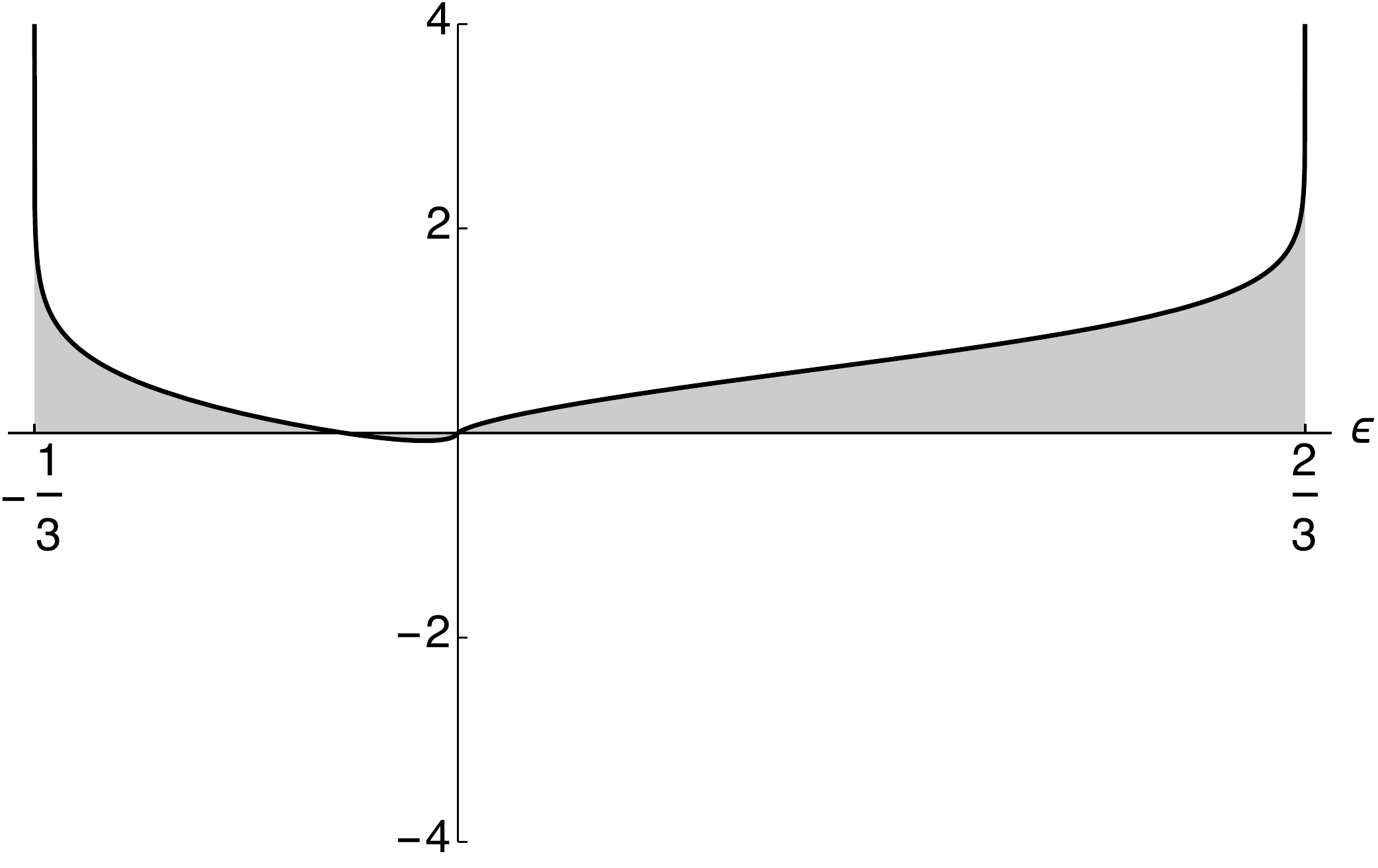}
  \caption{Harmonic mean ($\mathcal{H}$)}
  \label{fig:multihm}
\end{subfigure}%
\begin{subfigure}{.33\textwidth}
  \centering
  \includegraphics[width=\linewidth]{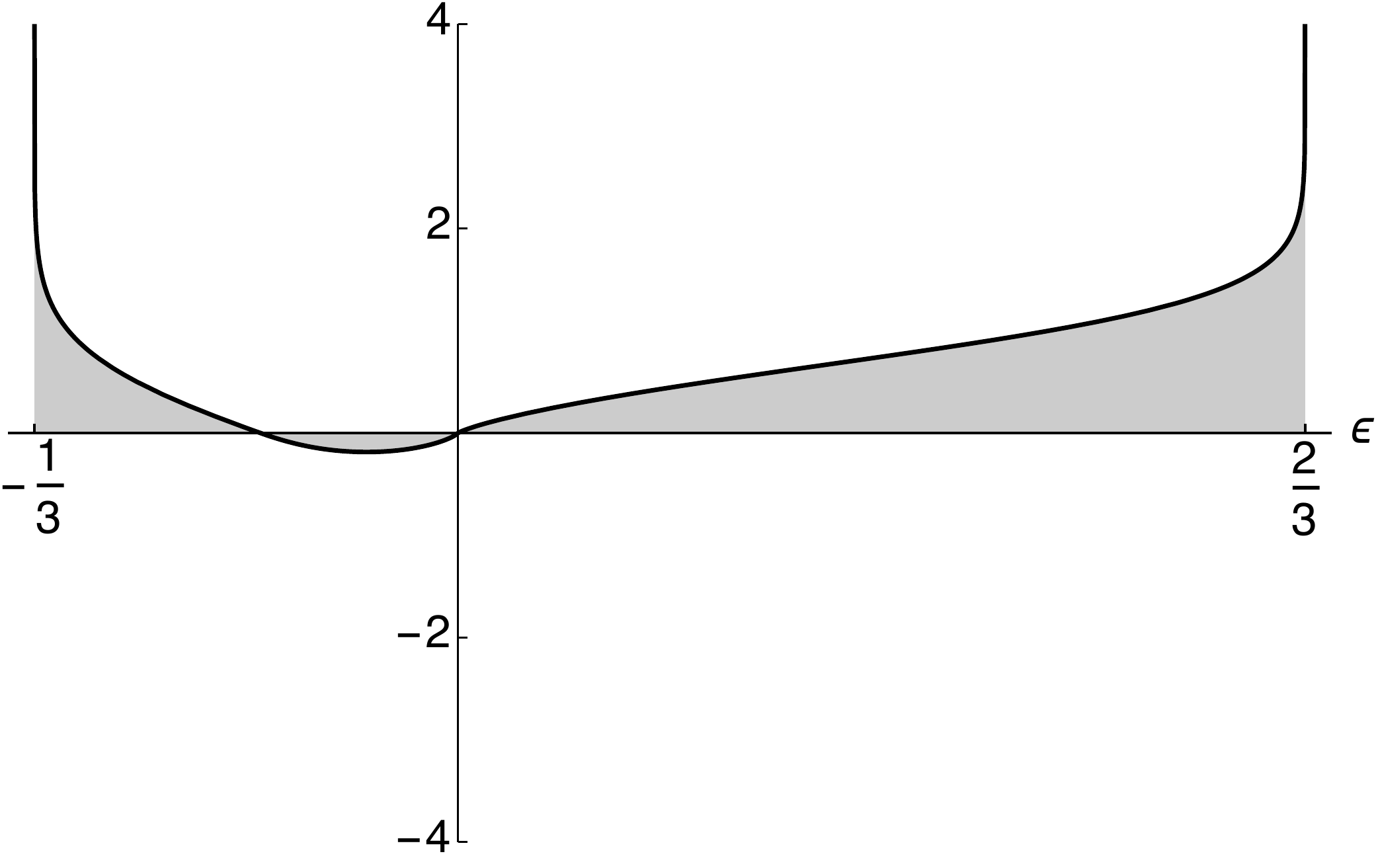}
  \caption{Minimum recall ($min$)}
  \label{fig:multimin}
\end{subfigure}
\caption{The influence function in ternary problems, $\mathds{I}^{\mathcal{S}}_3(\epsilon)$, for each global performance score throughout the range $-\frac{1}{3} \leq \epsilon \leq \frac{2}{3}$.}
\label{fig:multi2}
\end{figure*}

Several insights about {\bfseries the diverse behaviour of the performance scores} can be extracted from just a glimpse at Figures \ref{fig:binary1} and Figure \ref{fig:binary2}. First, regarding the {\it local performance scores} studied, they are completely different to each other; whilst the influence function for $\mathcal{P}^{1}$ shows a negative behaviour for all $\eta \neq 0.5$, $\mathds{I}^{\mathcal{R}^{1}}_2(\eta)$ is positive when $c_1$ is the minority class and takes negative values in the opposite case. Concerning the most common {\it global performance scores} of Figure \ref{fig:binary2}, two different major behaviours can be easily detected: on one hand, we have $\mathcal{A}cc$; its influence function shows a negative behaviour for the whole spectrum of unbalanced settings. On the other hand, we have the performance scores commonly used in unbalanced domains, $\mathcal{A}$, $\mathcal{G}$, and $\mathcal{H}$. These scores show a positive influence function for all the values of $\eta$ and they share a common shape; their lowest value is in the equiprobability, and, from there, they strictly increase to the extremes ($\eta\sim0$ or $\eta\sim1$), where they, finally, achieve an exponential growth rate. The two extreme H\"{o}lder means introduced in this paper can also be categorised into these two groups. While $max$ shares a similar behaviour to the $\mathcal{A}cc$, $min$ behaves closely to the scores utilised for unbalanced data. Lastly, it is also worth noticing the similarity between the shape of the influence functions for $\mathcal{P}^{1}$ and $\mathcal{A}cc$ in this binary setting. Despite the fact that one is a local score and the other is a global score, they share a close response to a changing class probability distribution.

Thereafter, we deal with the issue of determining the {\bfseries adequateness of the studied performance scores} by comparing the influence function of these scores to the representative case, defined in Proposition 1, of the influence function that an adequate performance score should have:

\noindent-- Although {\it local performance scores} are not sufficient to summarise the overall performance for losing the global picture, they give partial valuable advice: several global performance scores are just H\"{o}lder means \cite{Bul03} or other averaging functions over the local performances, e.g. $\mathcal{A}cc$, $\mathcal{G}$, etc \cite{San15}. In order to study the adequateness of these scores, we focus on the most problematic section of the plots: the values of $\eta < 0.5$. There, the class $c_1$ is the minority class, and due to the fact that the use of the \textsf{BDR} is assumed, the misclassification rate of $c_1$ considerably grows, on average, as $\eta$ decreases. As for $\mathcal{P}^{1}$, although it is extensively used in the literature, it is not adequate enough for this scenario since it does not penalise the decrease on the prediction power of the \textsf{BDR} for $c_1$ when this class becomes minority: $\mathds{I}^{\mathcal{P}^{1}}_2(\eta)$ takes negative values for $\eta < 0.5$ (Proposition 1). This is due to the fact that, for low values of $\eta$, the \textsf{BDR} achieves, for $\mathcal{P}^{1}$, higher values than in the balanced scenario. There, it only classifies examples as $c_1$ if they are ``undoubtedly'' drawn from that minority class, i.e. examples that are far away from the density of $c_2$. In such schemes, the ratio of the correctly classified examples over the predicted ones is considerably high, and it increases as the distribution grows in skewness. In the extreme, we have the convention $0/0=1$. On the contrary, $\mathcal{R}^{1}$ shows a more appropriate description of the effect of the class distribution on the resulting classifier; $\mathds{I}^{\mathcal{R}^{1}}_2(\eta)$ is positive for $\eta<0.5$. Therefore, regarding using precision or recall as parameters for a global measure function, two conclusions can be extracted: 
\begin{enumerate}[label=(\roman*)]
\item Unweighted H\"{o}lder means among the precisions are inadequate due to the fact that, as $c_2$ is a reflection of $c_1$, the values of $\mathds{I}^{\mathcal{S}}_2(\eta)$ for an average function will always be under the $x$-axis. Other factors are more influential in the score rather than a good prediction of the minority class.
\item Averaging recalls is a better choice since, for certain H\"{o}lder means, a positive value of the influence function for almost all $\eta$ will be shown. This is due to the fact that the positive part of $\mathds{I}^{\mathcal{S}}_2(\eta)$ for the recalls of $c_1$ and $c_2$ is always greater than the negative part.
\end{enumerate}

\noindent-- In relation to the {\it global performance scores}, our study supports the conclusion of the state-of-the-art literature stating that $\mathcal{A}cc$ is not an adequate score for unbalanced problems \cite{Gu09}. Here, Proposition 1 is not required to determine the insensitivity of the classification accuracy to the class-imbalance extent. Since its influence function shares shape with the inadequate function $\mathds{I}^{\mathcal{P}^{1}}_2(\eta)$, $\mathcal{A}cc$ can be directly appointed as an inadequate performance score. By extension, $max$ is not an appropriate score as well. The reason of the inadequateness of the latter is that it only takes into account the maximum recall, which usually coincides with the recall of the majority class due to the fact that the \textsf{BDR} favours this kind of classes. Contrastingly, the performance scores of the other behavioural group ($\mathcal{A}$, $\mathcal{G}$, $\mathcal{H}$ and $min$) are adequate to the class-imbalance problem since their influence functions meet the two conditions of Proposition 1. In this case, $\mathcal{A}$ is the score showing the lowest sensitivity -- its influence function has the smoothest shape --, which is followed by $\mathcal{G}$, then $\mathcal{H}$ and, finally, $min$. There it can be seen that, in these scores, the behaviour of the \textsf{BDR} is penalised when the class distribution is not balanced, i.e. the misclassification of the minority class also produces high drops in these performance scores. In the extremes, we discover that the situation of the \textsf{BDR} acting as a dummy classifier in situations of extremely skewed class distribution is strongly penalised. There, the influence functions exponentially grow as $\eta$ gets closer to either $0$ or $1$. In general terms, unweighted H\"{o}lder means over the recalls with $p>1$ will be inadequate to the class-imbalance extent since greater recalls have more presence in the score than lower recalls. On the contrary, means with $p<1$ would be adequate to determine the competitiveness of a classifier as lower recalls have more influence in the resulting score. This generalisation can be easily drawn from the H\"{o}lder mean inequality (Definition 1).

\subsubsection{Multi-class Problems}

Figure \ref{fig:multi2} presents results for 
accuracy, $\mathcal{A}cc$ (Figure \ref{fig:multiacc}), 
maximum recall, $max$ (Figure \ref{fig:multimax}), 
arithmetic mean, $\mathcal{A}$ (Figure \ref{fig:multiam}), 
geometric mean, $\mathcal{G}$ (Figure \ref{fig:multigm}), 
harmonic mean, $\mathcal{H}$ (Figure \ref{fig:multihm}), 
and minimum recall, $min$ (Figure \ref{fig:multimin}). 
The local performance scores are omitted in the multi-class scenario since the loss of the global picture becomes aggravated when more than two classes are used.  All these figures are similar to the binary functions but the $x$-axis shows, in the current case, the parameter $\epsilon$ over the range $-1/3 \leq \epsilon \leq 2/3$. There, the values $\epsilon<0$ represent the multi-majority version of the problem and  $\epsilon>0$ the multi-minority version. In view of these results, the conclusions for binary problems generalise well to the multi-class scenario: 

Regarding the {\bfseries diversity of the global performance scores}, the same two major groups can be perceived; $\mathcal{A}cc$ and $max$ on one hand, and the Pythagorean means and $min$ on the other. In the latter group, it can be seen that, due to the H\"{o}lder mean inequality, the scores can be also arranged by their sensitivity to the class-imbalance extent as $\mathds{I}^{\mathcal{A}}_K(\epsilon) \leq \mathds{I}^{\mathcal{G}}_K(\epsilon) \leq \mathds{I}^{\mathcal{H}}_K(\epsilon) \leq \mathds{I}^{min}_K(\epsilon)$.

Concerning the {\bfseries adequateness of the performance scores}, it can be concluded that:
\begin{enumerate}[label=(\roman*)]
\item Due to the fact that their influence functions violate the conditions of Proposition 1 ($\mathds{I}^{\mathcal{S}}_K(\epsilon)\leq 0, \forall \epsilon$), $\mathcal{A}cc$ and $max$ are inadequate scores for assessing unbalanced problems.
\item The unweighted Pythagorean means over the recalls and $min$ are adequate to assess the competitiveness of classifiers in this complex domain; their influence functions truly capture the hindering behaviours of the \textsf{BDR} (Proposition 1).
\end{enumerate}

It is also worth mentioning that the generalisation of the study for the whole set of H\"{o}lder means also applies well for multi-class problems. Moreover, it can be seen that achieving a high value of an adequate performance score for multi-minority problems is far more difficult than for multi-majority; $\mathds{I}^{\mathcal{S}}_K(\epsilon)$ is always higher for positive values of $\epsilon$. Finally, we want to point out the particularity that appears in Figure \ref{fig:multigm} ($\mathcal{G}$), in Figure \ref{fig:multihm} ($\mathcal{H}$) and in Figure \ref{fig:multimin} ($min$); $\mathds{I}^{\mathcal{S}}_K(\epsilon)$ is below $0$ for the negative values of $\epsilon$ near the balance situation, i.e. for $\epsilon \to 0^{-}$. In Section \ref{sec:st2}, we properly address this interesting particularity.

\section{Second Study: On the Inherent Scores of the Proposals for Unbalanced Domains}\label{sec:st2}

The fact that any learning task can be viewed as an optimisation problem leads us to the second unanswered question exposed in the introduction; {\it Since there are many different and diverse performance scores available, which performance scores are maximised in the most predominant learning solutions designed to deal with skewed classes? Are they the adequate performance scores detected in the first study?} To the best of our knowledge, little effort has been made in the literature towards answering these questions. Yet, the interpretation of previous works with regards to this issue seems to be in contradiction. On one hand, there are works claiming that there is no algorithmic solution to the class-imbalance problem since the \textsf{BDR} establishes a fundamental limit in the performance of any classifier \cite{Dru05}. The argument is that the \textsf{BDR} is asymptotically sought in this domain, i.e. all the solutions proposed in the class-imbalance domain can be categorised as traditional supervised learning approaches since they maximise $\mathcal{A}cc$ (by minimising the $0$-$1$ loss). So, therefore, no competitive classifier can be proposed for these problems. On the other hand, there are a large number of methodological contributions designed to overcome the intrinsic difficulties of this intricate domain. Moreover, most of these solutions report positive results \cite{He09}. This suggests that they produce classifiers which are maximised for a more appropriate score to the class imbalance scenario than  $\mathcal{A}cc$.

Therefore, in order to enlighten this apparent contradiction of statements and provide answers, we, now, theoretically scrutinise the state-of-the-art literature. Thus, we can devise which are the decision rules behind the proposals of the literature, their competitiveness and their inherently maximised performance scores.

\subsection{Major Solutions for the Class-imbalance Problem}

Concerning the unbalanced learning literature, most of its methodological contributions can be mainly categorised into four main kinds of approaches: (1) data sampling \cite{Bat04}, (2) cost-sensitive learning \cite{Liu06}, (3) algorithmic modification \cite{Gal12}, and (4) the use of ensembles \cite{Gal12}. In this paper, we study data sampling and cost-sensitive learning due to the following facts: (i) they dominate the current research efforts \cite{He09}. (ii) Both approaches are more transparent to the Bayesian decision theory since they never behave as blackboxes \cite{Rod14}. (iii) Additionally, there is a lack of a unified framework for the heterogeneous approaches of algorithmic modification and ensembles which impedes their categorisation and study \cite{Gal12}.

\subsubsection{Data Sampling}

By means of data sampling techniques, the training dataset is modified in order to provide a more balanced class distribution \cite{Sae16} so that, when classical supervised algorithms \cite{Wu07} are used in the learning process, the resulting classifiers are not biased towards the majority classes \cite{Bat04} \cite{Cha02}. In other words, this approach allows the traditional learning systems to learn from a {\it ``safe scenario''} where the probabilities of the classes hinder the learnt classifiers in an insignificant or null manner. The two simplest methods used are {\it random over-sampling} (ROS) and {\it random under-sampling} (RUS). While ROS balances the class distribution by the random replication of the examples of the minority classes of the training dataset, the balance distribution is achieved in RUS by the random removal of examples of the overrepresented classes. From a theoretical point of view, both methods are equivalent. They balance the class distribution up to having a uniform class distribution or, at least, a hardly noticeable unbalanced distribution. Formally, data sampling methods, instead of using the generative model as defined in eq. (\ref{genModel}), modify the training dataset in such a way they try to learn a classifier from the following model:
\begin{equation}
\rho'(\mathbf{x},c|\boldsymbol{\theta}) = p'(c)\rho(\mathbf{x}|c,\boldsymbol{\theta}).
\label{balanceModel}
\end{equation}
\noindent where $p'(c)$ is a multinomial distribution near the equiprobability, i.e. close to $\mathbf{e}$. Then, by directly applying the \textsf{BDR}, our surrogate of the traditional learning approach, over the modified generative model, we reach the following decision rule:
\begin{equation}
\label{datasamplingDR}
 \hat{c}_{B} = \arg\max_{i}\eta_i'\rho(\mathbf{x}|c_i,\boldsymbol{\theta}_i), \text{ where } \eta_i'\sim\frac{1}{K}.
\end{equation}
\noindent Thus, it can be concluded that the joint use of a data sampling technique and the \textsf{BDR} is practically equivalent (equal for $\boldsymbol{\eta}' = \mathbf{e}$) to use the equiprobable Bayes decision rule. The latter rule does not take into account the class distribution and it is defined as:

\begin{definition} Assuming $\rho(\mathbf{x}|c,\boldsymbol{\theta})$ and $\boldsymbol{\eta}$ to be known, the {\it equiprobable Bayes decision rule} (\textsf{EDR}) is given by
\begin{equation}
\label{EDRx}
 \hat{c}_{E} = \arg\max_{i}\rho(\mathbf{x}|c_i,\boldsymbol{\theta}_i).
\end{equation}
\end{definition}

Figure \ref{fig:sampling} summarises our theoretical reasoning behind the data sampling approach. However, when dealing with real-world problems, where the generative function is usually unknown, each method introduces its own set of problematic consequences that might hinder the learning task\cite{He09}; while ROS may produce over-fitting towards the minority classes, RUS may discard data which are potentially important to the classification process. In order to overcome these problems, some heuristic methods have been proposed in the literature; Tomek links \cite{Tom76}, condensed nearest neighbour rule (CNN) \cite{Har68}, one-side selection (OSS) \cite{Kub97}, synthetic minority over-sampling techniques (SMOTE) \cite{Cha02}, and combinations among them \cite{Bat04}. The main motivation behind some of these proposals is not only to balance the training data, but also to remove noisy examples lying on the wrong side of the decision region \cite{Bat04}. Due to its nature and despite  the fact that prior works \cite{He09}\cite{Gal12} have shown that data sampling is usually a positive practical solution, this methodology has mainly been criticised due to altering the original class distribution \cite{He09}, or even, the distribution of the feature space as happens with SMOTE.

\begin{figure}[t]
  \centering
  \includegraphics[width=9cm]{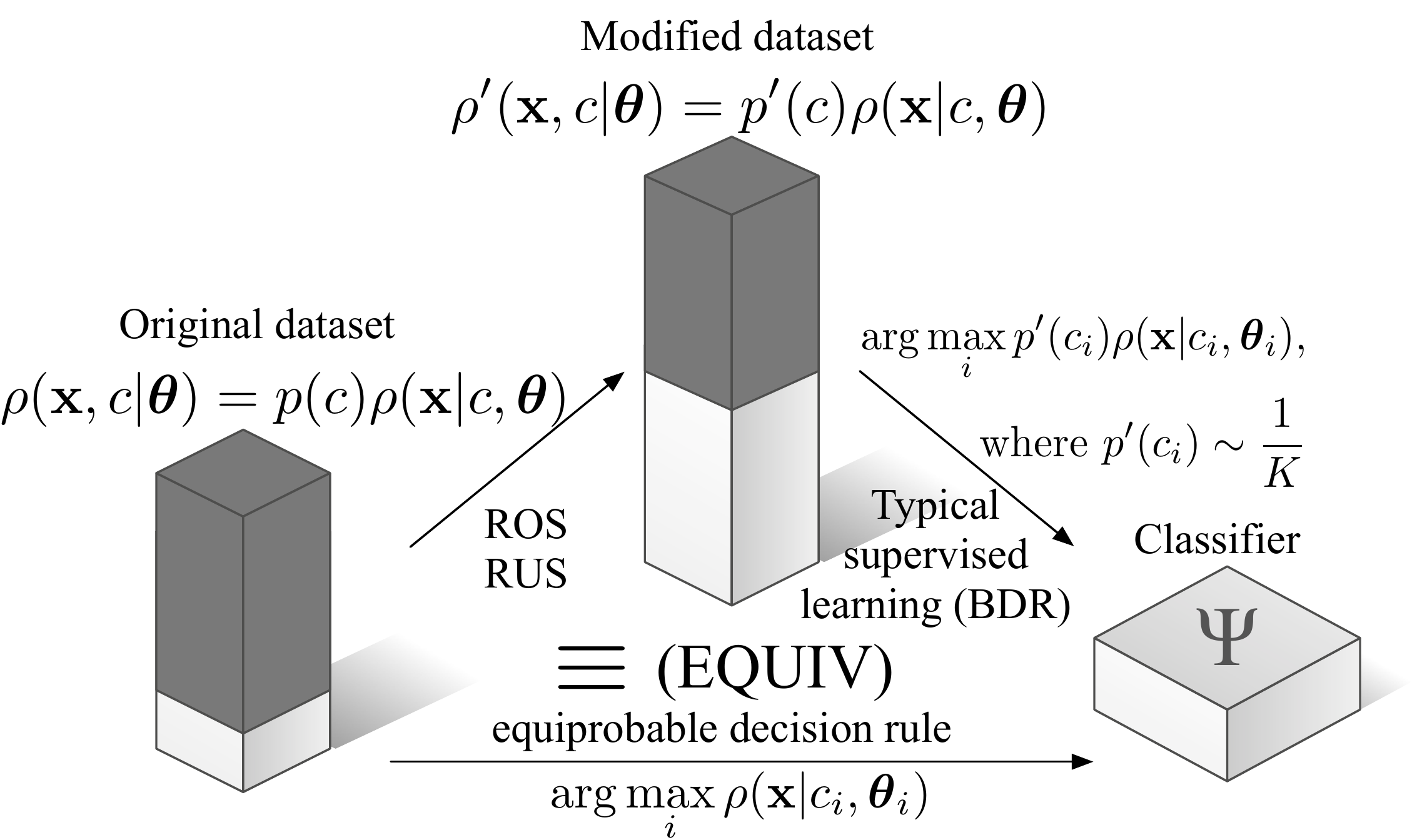}
  \caption{Data sampling techniques.}
  \label{fig:sampling}
\end{figure}

\subsubsection{Cost-sensitive Learning}

Cost-sensitive learning \cite{Kri11} is a paradigm which studies and solves classification problems where some types of misclassifications may be more crucial than others, e.g. rejecting a valid credit card transaction may cause an inconvenience while approving a large fraudulent transaction may have very negative consequences. Due to the fact that, in situations of disproportionate class probabilities, common learning systems tend to be overwhelmed by the majority classes ignoring the minority ones, the proposed cost-sensitive learning solutions usually assume that the cost of misclassifying an example from a minority class is higher than that of misclassifying an example from a majority class. Thereby, the ratio of correct classifications for the minority classes may be improved \cite{Liu06}\cite{Tha10}.

The concept of cost matrix appears in every cost-sensitive learning approach. The cost matrix is a numerical representation of the cost of classifying examples from one class to another. Formally, the cost matrix, $\mathbf{B} = [b_{i,j}]_{1 \leq i,j \leq K}$, is a square matrix of size $K$ where each element $b_{i,j}$ is the cost of classifying an example of class $c_i$ (rows) as class $c_j$ (columns). According to \cite{He09}, cost sensitive learning is superior to data sampling methods. However, although this framework can significantly improve the performance, it takes for granted the availability of a cost matrix and its associated cost items. Unfortunately, establishing a cost representation of a given domain can be particularly challenging and in some cases impossible. Moreover, this problem is exacerbated in the class-imbalance domain, where the misclassification costs are not intrinsic characteristics of the generative model but rather an artificial solution to seek an improvement in the classification of the minority classes. For that reason, several proposals in the literature appoint cost matrices for class-imbalance problems; e.g. \cite{Jap02} suggests the use of non-uniform error costs defined by means of the class-imbalance ratio presented in the dataset, i.e. $b_{i,j}= \eta_i/\eta_j$. Hence, the cost of misclassification of the underrepresented classes is higher than the overrepresented ones.  

Therefore, once the cost matrix is defined and assuming the generative model to be known, an example can be directly classified by means of the \textsf{BDR} for unequal misclassification costs\footnote{This rule differs from the \textsf{BDR} used in the manuscript in the fact that our version of the \textsf{BDR} (Definition 2) assumes equal misclassification costs.} (CS-\textsf{BDR})\cite{Ber85}. This rule minimises the expected misclassification costs as expressed in the cost matrix instead of a $0$-$1$ loss as the traditional solutions seek. 

Next, by assuming\footnote{We choose this method for convenience and clarity in the calculi. Other works, e.g. \cite{Liu06}, which are equivalent under our theoretical perspective, employ different methods to obtain these cost matrices so that the bias towards the majority group can be reduced\cite{Sae16}.} the method proposed by Japkowicz and Stephen \cite{Jap02} to establish a cost matrix for the class-imbalance problem, the CS-\textsf{BDR}, which is optimal for the H\"{o}lder mean with $p=1$ and $\zeta_i=W_i\eta_i$, can be formally defined as
\begin{equation}
 \hat{c}_{B} = \arg\max_{i}W_i\eta_i\rho(\mathbf{x}|c_i,\boldsymbol{\theta}_i), \text{ where }  W_i=\sum_{j =1}^K \frac{1}{b_{i,j}}.
 \label{costsensiDR1}
\end{equation}
\noindent In order to discover the actual decision rule under the proposal of \cite{Jap02}, and by extension to other cost-sensitive methods, we now just substitute the values of the $W_i$ in eq. (\ref{costsensiDR1}):
\begin{equation}
 \hat{c}_{B} = \arg\max_{i}\eta_i\sum_{j =1}^K \frac{1}{b_{i,j}}\rho(\mathbf{x}|c_i,\boldsymbol{\theta}_i) = \arg\max_{i}\rho(\mathbf{x}|c_i,\boldsymbol{\theta}_i).
 \label{costsensiDR2}
\end{equation}
\noindent As can be seen with this approach, the \textsf{EDR} is also reached.

\subsection{An Analysis of the Equiprobable Decision Rule}

In the previous paragraphs, it is shown that the asymptotic decision rule sought in most of the learning systems proposed in the literature for the class-imbalance scenario is the \textsf{EDR} (Definition 4). Additionally, we are confident enough to think that, also for the approaches of algorithm modification and ensembles, as long as their intention is to overcome the hindering behaviours expressed in Section 3.2 by reducing the bias of the traditional learners towards the majority classes, they will seek a decision rule which will be, if not the same, similar to the \textsf{EDR}. For that reason, in this section, we take a step further and study the main properties of this rule.

\subsubsection{Competitiveness in the Class-Imbalance Scenario}

First, we deal with the problem of determining whether the \textsf{EDR} is a competitive classifier to the class-imbalance domain. Opportunely, both the proposed influence function, eq. (\ref{imbfeat}), and the fact that the \textsf{EDR} can be viewed as a particular case of the \textsf{BDR} can be used to shed some light into this issue. Regarding the latter, we focus on the following relationship between both decision rules:

\begin{proposition} Let $\mathbf{A}_{\Psi}(\boldsymbol{\eta},\boldsymbol{\theta})$ be the true confusion matrix resulting from applying the algorithm $\Psi$ to a classification problem $\gamma_K$ with generative function equal to eq. (\ref{genModel}) and with parameters $\boldsymbol{\eta}$ and $\boldsymbol{\theta}$. Hence, when the generative model is known, for both the \textsf{BDR} ($\Psi=B$) and the \textsf{EDR} ($\Psi=E$), it holds true that
\begin{equation}
\forall \boldsymbol{\eta}, \mathbf{A}_{B}(\boldsymbol{e},\boldsymbol{\theta}) = \mathbf{A}_{E}(\boldsymbol{\eta},\boldsymbol{\theta}).
\label{relationCM}
\end{equation}
\end{proposition}

This proposition\footnote{Proofs for this proposition and its corollaries are not included in the manuscript due to their triviality; they can be easily inferred by using simple algebra from the proposed mathematical framework.} highlights the fact that, since the \textsf{EDR} is a particular case of the \textsf{BDR} (when the \textsf{BDR} is applied to the balanced version of $\gamma_K$), it results in a constant confusion matrix through the whole spectrum of the class distribution and its value is equal to the one resulting from the \textsf{BDR} in that particular case. Additionally, this relationship between both decision rules has an effect on the studied numerical performance scores:

\begin{corollary} Let $\mathcal{S}_{\Psi}(\boldsymbol{\eta},\boldsymbol{\theta})$ be a numerical performance score summarising the true confusion matrix  $\mathbf{A}_{\Psi}$. Then, 
\begin{equation}
\mathcal{S}_{B}(\boldsymbol{e},\boldsymbol{\theta}) = \mathcal{S}_{E}(\boldsymbol{e},\boldsymbol{\theta})
\label{relationPS}
\end{equation}
\noindent is true in all circumstances. 
\end{corollary}

\begin{corollary} When a numerical performance score $\mathbf{S}_{\Psi}$ summarises the true confusion matrix $\mathbf{A}_{\Psi}$ with independence\footnote{No class probability is used in the calculation of the value of the score.} of the class distribution $\boldsymbol{\eta}$, the relation of Proposition 2 can also be held to be true for $\mathcal{S}_{\Psi}$. That is,
\begin{equation}
\forall \boldsymbol{\eta}, \mathcal{S}_{B}(\boldsymbol{e},\boldsymbol{\theta}) = \mathcal{S}_{E}(\boldsymbol{\eta},\boldsymbol{\theta}).
\label{relationAdeqPS}
\end{equation}
\end{corollary}

Note that, in our framework, the numerical performance scores whose summary is independent to $\boldsymbol{\eta}$ are the local scores (recall and precision), the unweighted H\"{o}lder means among those local scores and the weighted H\"{o}lder means among those local scores whose weights are not calculated on the class distribution.

\begin{corollary} The value of classification accuracy, a score which is summarised using the distribution of the classes, for the \textsf{BDR} on the balanced version of a problem is equal to the value of the arithmetic mean among the recalls obtained by the \textsf{EDR} in $\gamma_K$ showing any value of class imbalance. Formally,
\begin{equation}
\forall \boldsymbol{\eta}, \mathcal{A}cc_{B}(\boldsymbol{e},\boldsymbol{\theta}) = \mathcal{A}_{E}(\boldsymbol{\eta},\boldsymbol{\theta}).
\label{relationAcc}
\end{equation}
\end{corollary}
Now, by substituting eq. (\ref{relationAdeqPS}) in eq. (\ref{imbfeat}); the influence function for the adequate performance scores can be rewritten as:
\begin{equation}
\mathds{I}^{\mathcal{S}}_K(\boldsymbol{\eta})= \displaystyle\int\limits_{\boldsymbol{\Theta}}{}[\mathcal{S}_{E}(\boldsymbol{\eta},\boldsymbol{\theta})-\mathcal{S}_{B}(\boldsymbol{\eta},\boldsymbol{\theta})]d\boldsymbol{\theta}.
\label{imbfeat2}
\end{equation}
By means of this transformation, this new version of the influence function can be used to study whether the \textsf{BDR} or the \textsf{EDR} has, on average, a superior behaviour with regards to a determined adequate performance score ($\mathcal{A}$, $\mathcal{G}$, $\mathcal{H}$, and $min$). If the influence function is positive for the whole range of $\boldsymbol{\eta}$, then the \textsf{EDR} behaves, on average, better for that performance score than the \textsf{BDR}. The opposite case, a whole negative influence function, will show that the \textsf{BDR} is superior to the \textsf{EDR}. From just a re-examination of Figure \ref{fig:multi2} (\ref{fig:multiam}-\ref{fig:multimin}), where the influence function is studied in ternary problems and displayed for the performance scores, the following conclusion can be extracted: since $\mathds{I}^{\mathcal{S}}_K(\boldsymbol{\eta})$ for the adequate scores is positive for ``almost'' the whole range of $\boldsymbol{\eta}$, {\it the \textsf{EDR} is, in general and on average, more competitive than the \textsf{BDR}}. Therefore, the \textsf{EDR} tears apart the fundamental limit in the performance of every algorithm solution established by the \textsf{BDR} and claimed in \cite{Dru05}. 

Finally, we want to remark that, although in theory there is no difference between dealing with binary or multi-class domains, in real world problems, most of the solutions seeking the \textsf{EDR} for unbalanced problems have been shown to be less effective or even to cause a negative effect in dealing with multiple classes \cite{Wan12}\cite{Sae16}.

\subsubsection{Optimality of the \textsf{EDR}}

In order to determine the performance score optimised in the \textsf{EDR}, we also take advantage of the definition of the influence function and the relationships between both decision rules:
\begin{thm} The equiprobable Bayes decision rule is optimal for $\mathcal{A}$, the unweighted H\"{o}lder mean among the recalls with $p=1$.
\end{thm}
\begin{proof}
Let the classification accuracy and the unweighted arithmetic mean among the recalls obtained by a classifier $\Psi$ in a classification problem $\gamma_K$ with a generative function determined by eq. (\ref{genModel}) and with parameters $\boldsymbol{\eta}$ and $\boldsymbol{\theta}$ be defined as
$$\mathcal{A}cc_\Psi(\boldsymbol{\eta},\boldsymbol{\theta})=\sum_{i=1}^K\eta_i\mathcal{R}_\Psi^i\text{, and }\mathcal{A}_\Psi(\boldsymbol{\eta},\boldsymbol{\theta})=\sum_{i=1}^K\frac{1}{K}\mathcal{R}_\Psi^i\text{, respectively.}$$
\noindent Also let the \textsf{BDR} and \textsf{EDR} be denoted as $\Psi=B$ and $\Psi=E$, respectively. Then, since the \textsf{BDR} obtains the optimal classification accuracy in every problem, it holds that:
$$\forall(\Psi,\boldsymbol{\eta},\boldsymbol{\theta}), \mathcal{A}cc_B(\boldsymbol{\eta},\boldsymbol{\theta}) = \max\{\mathcal{A}cc_\Psi(\boldsymbol{\eta},\boldsymbol{\theta})\}.$$
\noindent If the \textsf{BDR} is applied to the balanced version of the classification problem $\gamma_K$, i.e $\forall i, \eta_i= K^{-1}$, we obtain
$$\forall(\Psi,\boldsymbol{\eta},\boldsymbol{\theta}), \mathcal{A}cc_B(\boldsymbol{e},\boldsymbol{\theta}) = \max\{\sum_{i=1}^K\frac{1}{K}\mathcal{R}_\Psi^i\}.$$
Next, by the definition of $\mathcal{A}$, the $\max$ function can be rewritten as
$$\forall(\Psi,\boldsymbol{\eta},\boldsymbol{\theta}), \mathcal{A}cc_B(\boldsymbol{e},\boldsymbol{\theta}) = \max\{\sum_{i=1}^K\frac{1}{K}\mathcal{R}_\Psi^i\}=\max\{\mathcal{A}_\Psi(\boldsymbol{\eta},\boldsymbol{\theta})\}.$$
Finally, by Corollary 3, we reach the conclusion that the \textsf{EDR} optimises the unweighted arithmetic mean among the recalls:
$$\forall(\Psi,\boldsymbol{\eta},\boldsymbol{\theta}), \mathcal{A}_E(\boldsymbol{\eta},\boldsymbol{\theta}) = \max\{\mathcal{A}_\Psi(\boldsymbol{\eta},\boldsymbol{\theta})\}.$$
\end{proof}

This theorem manifests that the \textsf{EDR} achieves the lowest upper bound for $\mathcal{A}$, an adequate performance to the class imbalance extent, of any classifier. Therefore, most of the practical contributions to the class-imbalance scenario inherently also maximise this performance score. Unfortunately, this decision rule is not optimal for the other adequate performance scores ($\mathcal{G}$, $\mathcal{H}$ and $min$). This non-optimality can be straightforwardly proven using the influence function as expressed in eq. (\ref{imbfeat2}); both positive and negative values for $\mathds{I}^{\mathcal{S}}_K(\boldsymbol{\eta})$ will prove that neither the \textsf{BDR} nor the \textsf{EDR} always behave, on average, better than the other. Then, by just having a look at the previously indicated singularities of Figure \ref{fig:multigm} ($\mathcal{G}$), Figure \ref{fig:multihm} ($\mathcal{H}$) and Figure \ref{fig:multimin} ($min$), it can be seen that neither the \textsf{EDR} nor the \textsf{BDR} are optimal for these performance scores. For greater values of $K$ than in Figure \ref{fig:multi2}, the non-optimality holds, yet, the absolute values of these singularities are smaller. By having a glance at Figure \ref{fig:binary2}, we cannot say whether \textsf{EDR} is optimal for these performance scores in binary domains. However, by relaxing the geometry assumption\footnote{In Figure \ref{fig:binary2}, the behaviour of the \textsf{EDR} is, in fact, optimal for any unweighted H\"{o}lder mean with $p \leq 1$. This is due to the fact of that the geometry created in the model has equal variances and the same distance between adjacent means and the H\"{o}lder mean equality; the overlapping area for each feature space is equal for both classes. Unfortunately, even in this geometric scenario, for $p > 1$, the \textsf{EDR} is not optimal.} of equal unit variances in our model, we reach the same conclusion as in multi-class; {\it in general, for the unweighted H\"{o}lder means among the recalls with $p \neq 1$, \textsf{EDR} is not an optimal decision rule.} 

\subsection{Discussion on Maximising Other Scores beyond $\mathcal{A}$}

Several ideas regarding whether the remaining unweighted H\"{o}lder means among the recalls can be used to direct the definition of the forthcoming learning solutions for the class-imbalance domain can be extracted from the analyses performed in previous sections:

\begin{enumerate}[label=(\roman*)]
\item Provided that learning algorithms grouped by the \textsf{EDR} report positive results in the literature, and that the H\"{o}lder means with $p < 1$  have a higher sensitivity to the imbalance extent than $\mathcal{A}$, visibly, classifiers more adequate to deal with the class-imbalance related problems can be proposed\footnote{Since they have greater values for $\mathds{I}^{\mathcal{S}}_K(\boldsymbol{\eta})$, a classifier maximising these scores will behave, on average, better than a classifier maximising $\mathcal{A}$.}. In this range of $p$, (optimal) classifiers maximising $\mathcal{G}$, $\mathcal{H}$ and $min$, among other means, could be proposed. The optimal classifier for the latter will be the most restrictive that can be defined in this framework due to the fact that it ensures than the minimum recall of all classes must be the highest. Probably, a desired property (inferred from \cite{Fer11}) within this complex domain. 
\item On the contrary, the use of H\"{o}lder means with $p > 1$ in the definition of learning systems is not an adequate solution. These scores favour the greater recalls, one class will always have more chances to be selected in the classification. Moreover, this class rarely coincides with the minority class. In the extreme, we have the unweighted H\"{o}lder mean with $p=\infty$, whose optimal classifier is the one which classifies every instance to just one class. Therefore, this set of performance scores must be avoided to define loss functions in unbalanced domains, or at least, they must be used with special care. 
\end{enumerate}

In conclusion, any practical classifier $\Psi$ resulting from the maximisation of a H\"{o}lder mean with $p \leq 1$ over a training sample is a better option than a classifier learnt in the traditional supervised framework. These classifiers will have less probability of, in cases of skewed class distributions, behaving like a dummy classifier. Therefore, {\it optimising adequate scores\footnote{The interested reader can find, in the appendix, an example of how the decision regions for the different decision rules optimising the studied global performance scores are located in a ternary problem generated from a Gaussian mixture model.} to the class-imbalance extent may produce propitious classifiers reactive to the class-imbalance related problems.}

\section{Third Study: Bounds for the Competitiveness of a Classifier in the Class-Imbalance Scenario}\label{sec:st3}

Virtually, every paper proposing a class-imbalance solution has the exact same experimental setup \cite{Pra15}: A proposed method is compared against one or two previously proposed methods over a dozen or so datasets. Although this experimental setup is reasonable to support an argument that the new method is as good as or better than the state of the art, it still leaves many unanswered questions. Among them, we can find the question of whether the proposed solution is able to produce competitive classifiers: if the precedent solutions are not competitive (in terms of our definition), the proposal might be uncompetitive as well. For that reason, in this study, we focus on the last question of the introduction: {\it Can bounds guaranteeing the competitiveness of a classifier be provided for the value of certain adequate performance scores?} In particular, since they have gained notorious importance throughout the paper, we focus on the values of the unweighted H\"older means among the recalls. We refer to this value as $\mathcal{S}_{\Psi}^p$, where $p$ stands for the exponent of the mean used to assess a given classifier $\Psi$.

Next, in order to answer the question, we rewrite the definition of ``competitiveness of a classifier'' as {\it a competitive classifier is a classifier whose expected behaviour is superior to the expected behaviour of an already known baseline classifier}. Thus, by just instantiating  the expressions ``{\it expected behaviour}'' and ``{\it baseline classifier}'', and after some algebra, practical bounds for $\mathcal{S}_{\Psi}^p$ can be given in order to ensure the adequateness of $\Psi$ for the given unbalanced problem. Regarding the former concept, the expected behaviour cannot be determined by the direct use of the inherently maximised performance score due to the fact it would not be legitimate;  $\Psi$ might be favoured in the comparison. For that reason, we rely on common sense and on previous experience on the class-imbalance domain to define the term. Within this domain, it is interesting to obtain classifiers achieving great recalls for the minority classes while maintaining adequate recalls for the majority ones \cite{Gu09}, a fairly complicated task. Therefore, the competitiveness of the target classifier can be translated into obtaining greater or equal recalls for both the minority and majority classes than a baseline classifier for which prior knowledge is available. Finally, concerning the second term, we make use of the most-utilised base classifier for establishing the lower expected behaviour than a competitive classifier must obtain; the classifier representing the random guessing. Formally:

\begin{definition} The classifier representing the  random guessing of $K$ different classes is known as the uniformly random classifier (\textsf{RAND}) and it is given by
\begin{equation}
\label{RAND}
 \hat{c}_{R} = \textrm{Unif}\{1,K\}.
\end{equation}
\noindent where $\textrm{Unif}\{1,K\}$ is the discrete uniformly random function which assigns a categorical class $c_i$ to an example $\mathbf{x}$ with probability $1/K$. All the performance scores contemplated in Table \ref{scores} assign the same value to the goodness of this classifier: 
 \begin{eqnarray}
 \begin{tabular}{c}
$\mathcal{P}^i=\mathcal{R}^i = \frac{1}{K}, \forall i\text{, and }\mathcal{A}cc=\mathcal{A}=\mathcal{G}=\mathcal{H}=max=min=\frac{1}{K}.$\nonumber
\end{tabular}
\end{eqnarray}
In general, for the H\"{o}lder means among the recalls\footnote{This can be trivially proved by following a similar reasoning to Theorem 1 of \cite{Ort16} but using the H\"{o}lder mean equality instead.} (independently of the values of weights, $\boldsymbol{\zeta}$, and the exponent $p$), it holds that
$$M_p(\mathbf{R},\boldsymbol{\zeta})= \frac{1}{K}, \text{ where } \mathbf{R}=(\mathcal{R}^1,\ldots,\mathcal{R}^K).$$ 

\end{definition}

Now, by means of the previous instantiations, we can determine the {\it competitiveness} of $\Psi$ by just checking which values for an unweighted H\"{o}lder mean ensure that a recall of at least $1/K$ is obtained for all classes. Opportunely, this calculation can be easily performed in this kind of functions. Then, let $S^p_{\text{sup}}$ be defined as the lowest value for $\mathcal{S}_{\Psi}^p$ ensuring that the classifier $\Psi$ is certainly superior to \textsf{RAND}. This extreme situation takes place when just one recall is equal to $1/K$ and the remaining are all equal to $1$, i.e. $\exists!j(\mathcal{R}^j=1/K \wedge \forall i \neq j, \mathcal{R}^i=1)$. In this scenario, the slightest negative variation in the score could result in an incompetent classifier, i.e. a classifier $\Psi$ reporting a score of $\mathcal{S}_{\Psi}^p=S^p_{\text{sup}} - \epsilon$, where $\epsilon \to 0^{+}$, could indicate that $\exists i, \mathcal{R}^i<1/K$. On the contrary, a score of $S^p_{\text{sup}} + \epsilon$ will always indicate that $\Psi$ is a competitive classifier with $\forall i, \mathcal{R}^i>1/K$. Thus, by just substituting these recalls in the definition of a H\"{o}lder means -- eq. (\ref{holdermean}) --, we obtain
 \begin{eqnarray}
 \begin{tabular}{lll}
$S^p_{\text{sup}} =$ &$\displaystyle\Big(K^p+ K^{(p-1)} +1\Big)^{\frac{1}{p}}.$
\label{superiorlowerlimit}
\end{tabular}
\end{eqnarray}

Regarding the opposite scenario, the {\it incompetence} is determined by calculating which values of the H\"{o}lder means ensures that at least one recall is below the random guessing value, $1/K$. Here, let $S^p_{\text{inf}}$ be the strictly upper value for the score $\mathcal{S}_{\Psi}^p$ indicating that $\Psi$ is not competitive, i.e. it is inevitably inferior to \textsf{RAND}. For this case, we choose the scenario of obtaining a classifier behaving in the same manner as the \textsf{RAND}. Therefore, any negative variation in the score of $\Psi$ will indicate the incompetence of it. Analogously, any positive variation might indicate that the classifier is competent. In the latter assertion we cannot remove the model verb 'might' due to the fact that a higher score, but less than $S^p_{\text{sup}}$, is not sufficient for safeguarding better recalls than the \textsf{RAND} for all classes. Hence, the value for this upper value is:
\begin{eqnarray}
\begin{tabular}{cl}
$S^p_{\text{inf}} =$ &$\displaystyle\frac{1}{K}$ [= \textsf{RAND}]
\label{inferiorupperlimit}
\end{tabular}
\end{eqnarray}

\begin{figure}[t] 
\centering
\begin{subfigure}{.495\textwidth}
  \centering
  \includegraphics[width=\linewidth]{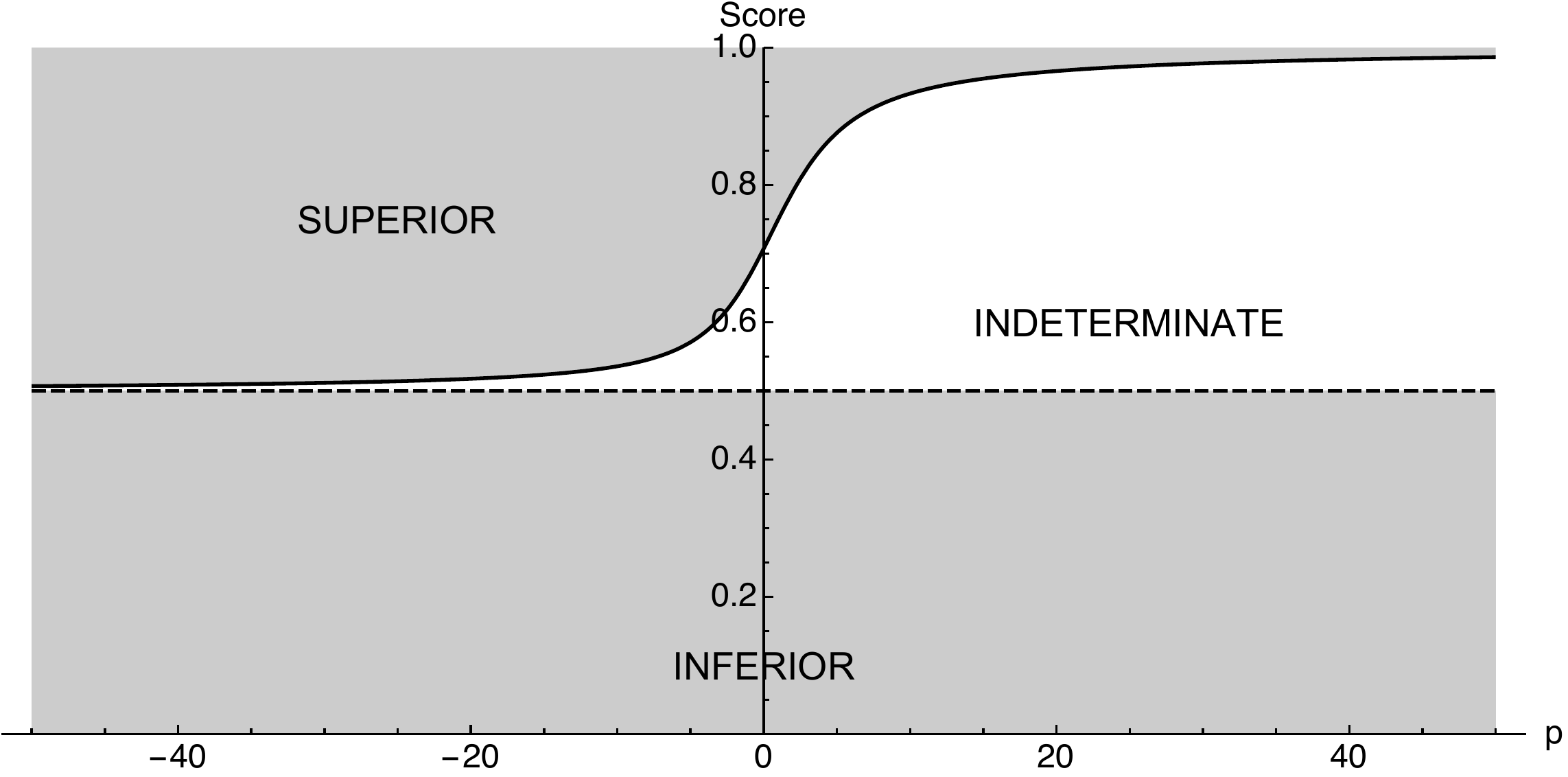}
  \caption{Binary classifier}
  \label{fig:binaryC}
\end{subfigure}
\begin{subfigure}{.495\textwidth}
  \centering
  \includegraphics[width=\linewidth]{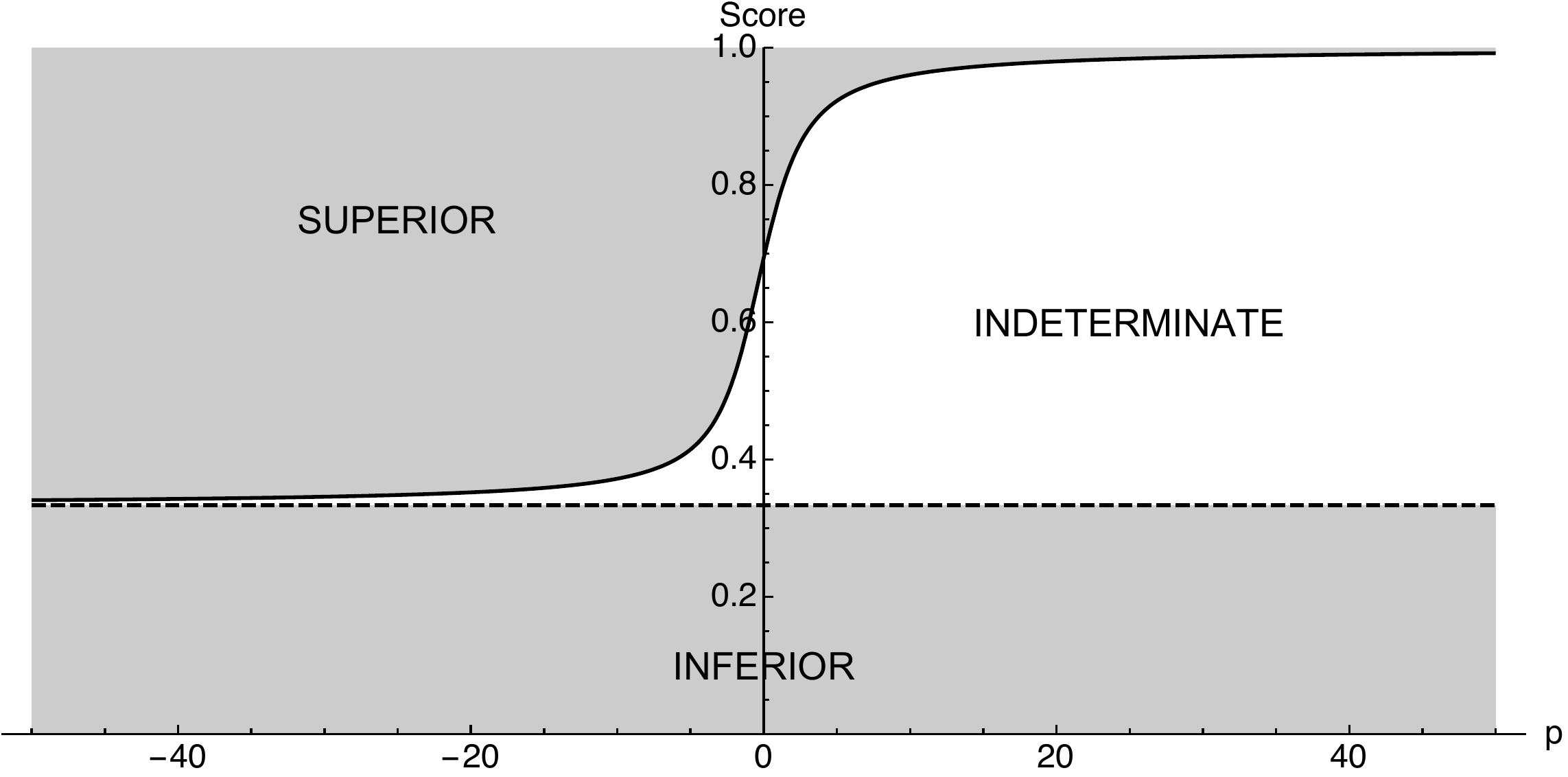}
  \caption{Ternary classifier}
  \label{fig:ternaryC}
\end{subfigure}
\caption{Limiting values for the unweighted H\"{o}lder means ensuring that a given classifier is superior/inferior to the random guessing.}
\label{fig:competitiveness}
\end{figure}

As a summary of the previous paragraphs, we plot Figure \ref{fig:competitiveness}. Here, the values for both eq. (\ref{superiorlowerlimit}), $S^p_{\text{sup}}$, and eq. (\ref{inferiorupperlimit}), $S^p_{\text{inf}}$, ($y$-axis) in the domain  $p \in [-50,50]$ ($x$-axis) are presented. Figure \ref{fig:binaryC} represents these values for binary problems and Figure \ref{fig:ternaryC} for ternary cases. In the figures, three different areas can be seen. They correspond to the cases for $\mathcal{S}_{\Psi}^p$ argued in this section:
\begin{eqnarray}
\begin{tabular}{rc}
if $\mathcal{S}_{\Psi}^p \in [S^p_{\text{sup}},1]:$ 			&SUPERIOR. $\Psi$ is a competitive classifier for the problem\\
if $\mathcal{S}_{\Psi}^p \in [0,S^p_{\text{inf}}):$				&INFERIOR. $\Psi$ is incompetent to solve the problem\\
if $\mathcal{S}_{\Psi}^p \in [S^p_{\text{inf}},S^p_{\text{sup}}):$	&INDETERMINATE. The competitiveness of $\Psi$ cannot be checked with just $\mathcal{S}_{\Psi}^p$\nonumber
\end{tabular}
\end{eqnarray}

In conclusion, it can be seen that the thesis highlighted throughout the whole paper is supported once more; the adequateness of the unweighted H\"{o}lder mean among the recalls with $p \leq 1$ to assess the goodness of a classifier in class-imbalance domains. Not only do they focus on the behaviour of all recalls, independently of their class probability values, but also the resulting score value is acutely informative in terms of competitiveness of the classifier. Here again, $min$ is the most restrictive and the most informative score: a value above $1/K$ indicates that the classifier is competitive. A value below denotes the opposite.

\section{Summary}\label{sec:sum}

Most of the existing learning algorithms are designed to asymptotically converge to the \textsf{BDR} by minimising the $0$-$1$ loss or, what is equivalent, maximising the classification accuracy. Unfortunately, when dealing with unbalanced problems, the classification accuracy is not an adequate performance score due to the fact that the underrepresented classes have very little impact on the measure when they are compared to the overrepresented classes. Therefore, it is imperative to define more adequate learning systems which can effectively deal with skewed class distributions. Thus, in order to shorten the distance towards the previous ideal, in this paper, an exhaustive analysis of a set of numerical performance scores is carried out, not only to be able to determine their adequateness to assess the goodness of a classifier in this complex scenario, but also to be capable of studying whether they are suited for being used to produce more competitive learning systems for unbalanced problems. Specifically, we focus on performing an exhaustive analysis of the most-common numerical performance scores used in the class-imbalance domain which can also be represented as H\"{o}lder means \cite{Bul03} among the recalls of all the classes. This set groups well-known performance scores such as accuracy, a-mean, g-mean etc. Since many interdependent factors may hinder the discerning skill of the classifier and, therefore, vary the value of the score, we develop a novel classification framework which allows us to marginalise out the class-imbalance component from the rest of the factors. As a result, the influence of the class-imbalance extent on the performance score using the long studied \textsf{BDR} as a classifier can be measured in isolation. With this study, we provide answers to the following questions:

{\bfseries Which performance scores are adequate to determine the competitiveness of a classifier in unbalanced domains?} {\it The performance scores which are unweighted H\"{o}lder means with $p \leq 1$ (a-mean, g-mean, h-mean, etc.) among the recalls are the most appropriate to evaluate the competitiveness of classifiers in unbalanced problems. In these cases, misclassifying the least probable classes is highly penalised.}

{\bfseries Which performance scores are maximised in the most common learning solutions designed to deal with skewed classes?} {\it Most of the learning solutions proposed in the unbalanced literature are designed to maximise the a-mean (which is the unweighted H\"{o}lder mean with $p = 1$) due to the fact that they asymptotically converge to the \textsf{BDR} for equiprobable classes.}

{\bfseries Can bounds guaranteeing the competitiveness of a classifier be provided for certain adequate performance scores?} {\it Yes, we finalise the paper by providing two different practical bounds for the performance scores expressed as unweighted H\"{o}lder means among the recalls with $p  \in \mathds{R} \cup \{+\infty,-\infty\}$; a bound for the lowest value of the performance score ensuring a competitive solution for unbalanced problems and a bound for the highest value of the score indicating an incompetent solution.}

Concerning future work, since most of the research efforts in unbalanced learning are specific algorithms and/or case studies, we mainly focus on the {\it potential theoretical research lines} due to the fact that, so far, only a limited amount of theoretical understanding on the principles and consequences of the class-imbalance problem have been addressed \cite{He09}. This study can easily be extended in several ways: first, as we only deal with numerical performance scores, similar studies can be proposed using other interesting kinds of scores; graphical performance scores \cite{San15}, adjusted (to the class-imbalance domain) performance scores \cite{Lop13}, or even, the whole confusion matrix \cite{Koc13}. Secondly, several other classifiers rather than the \textsf{BDR} can be used as a representative classifier in the study. Thirdly, analytical solutions can be found for either the influence function for each score or for the decision rules optimising H\"{o}lder means with $p < 1$. Fourthly, here, we exclusively deal with the supervised learning framework, however, other types of learning scenarios such as semi-supervised learning \cite{Ort16} can be studied. Lastly, we think that research on how the degree of imbalance of the real-world multi-class problems can be measured could also be interesting. While in binary problems, it can be measured by just using the imbalance-ratio \cite{Lop13}, this measure is insufficient to capture the complexity of a $K > 2$ class distribution vector $\boldsymbol{\eta}$. An informative single value measure will ease the comparison among class-imbalance problems with a different number of classes.

 \section*{Acknowledgments}
 This work is partially supported by the Basque Government (IT609-13 and Elkartek BID3A) and the Spanish Ministry of Economy and Competitiveness (TIN2013-41272P). Moreover, Jonathan Ortigosa-Hern\'andez is partly financed by the Ministry of Science and Innovation (MEC-FPU AP2008-00766) and Jose A. Lozano by both the Basque Government (BERC program 2014-2017) and the Spanish Ministry of Economy and Competitiveness (Severo Ochoa Program SEV-2013-0323).
 
 \bibliographystyle{plain}
\bibliography{mylib}{}

\newpage
\appendix
\section{Case Study: How are the decision regions for the classifiers optimising the studied scores located in a ternary problem?}

With the purpose of exposing the asymptotical behaviour of the classifiers maximising the numerical performance scores studied in the manuscript, here, we set up a controlled example showing how the optimal classifiers for those scores split the feature space into different decision regions. As we assume the generative model to be known, here, each optimal classifier will be referred to as the decision rule optimising a certain numerical performance score.

Deliberately, we define the example in terms of the framework of the manuscript: Let the example $\gamma_3$ be a ternary classification problem with a generative model composed of a univariate Gaussian mixture model represented by the following joint probability density function
\begin{equation}
\label{jointgenmodel}
f(\mathbf{x},c|\boldsymbol{\theta}) = \sum_{i=1}^{K} \eta_if(\mathbf{x}|c_i,\boldsymbol{\theta}_i){{\mathbb 1}(c= c_i)}.
\end{equation}
Note that the previous equation uses $f$ instead of $\rho$ (as in the manuscript) for the distribution of the feature space. This is due to the fact that, here, we assume the features to be continuous. Let ${\mathbb 1}(c= c_i)$ stand for the indicator function, i.e. it is equal to $1$ if  $c= c_i$ and $0$ otherwise. Then, let us instantiate each mixture component in this example as a univariate Gaussian distribution with parameters
\begin{eqnarray}
\begin{tabular}{ccc}
$f(x|c_1,\boldsymbol{\theta}_1) \sim N(3,0.5)$ &$f(x|c_2,\boldsymbol{\theta}_2) \sim N(5,0.5)$ &$f(x|c_3,\boldsymbol{\theta}_3) \sim N(6,0.5)$\nonumber
\end{tabular}
\end{eqnarray}
\noindent and let the class distribution be set as $\boldsymbol{\eta}=(0.6,0.3,0.1)$. Then, Figure \ref{fig1} shows the generative model of the proposed example. Whilst Figure \ref{fig:unbalanceP} plots the density distribution for each class multiplied by its class probability ($\eta_if(\mathbf{x}|c_i,\boldsymbol{\theta}_i)$), Figure \ref{fig:balanceP} presents only the mixture density distribution of the classes ($\eta_if(\mathbf{x}|c_i,\boldsymbol{\theta}_i)$). The main objective of presenting two views of the same model is two-fold; (i) the influence of the class distribution on the intricacy of the generative model can be perceived at a glance, and (ii) each view is a key factor in showing the behaviour of the two main decision rules studied in the manuscript; the \textsf{BDR} (which optimises $\mathcal{A}cc$) and the \textsf{EDR} (which is optimal for $\mathcal{A}$). 

\begin{figure}[h]
        \begin{subfigure}{\textwidth}
        \centering
                \includegraphics[width=\textwidth]{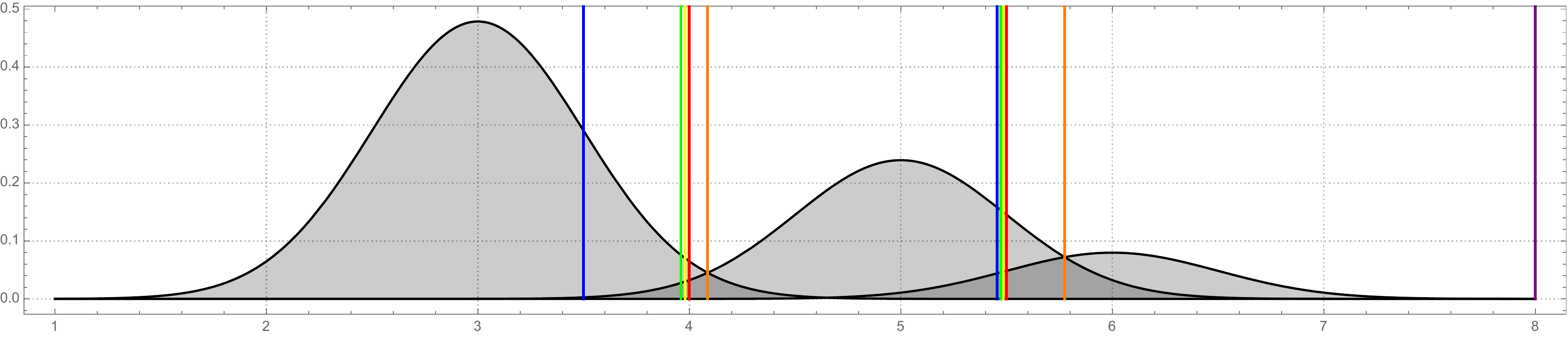}
                \caption{The generative model of the example as it is defined, i.e. the density distribution for each class is multiplied by its class probability.}
                \label{fig:unbalanceP}
        \end{subfigure}
        \begin{subfigure}{\textwidth}
         \centering
                \includegraphics[width=\textwidth]{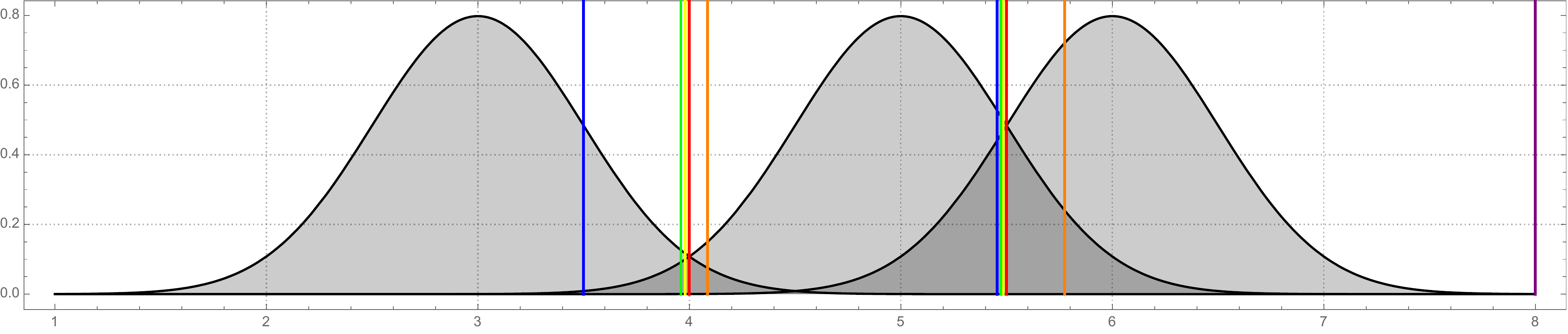}
                \caption{The density distribution of each categorical class.}
                \label{fig:balanceP}
        \end{subfigure}
        \caption{The decision region limits for the decision rules optimising the studied scores: \textsf{EDR} (red), \textsf{BDR} (orange), $\mathcal{G}$\textsf{-DR} (yellow), $\mathcal{H}$\textsf{-DR} (green), $\min$\textsf{-DR} (blue) and $\max$\textsf{-DR} (purple).}\label{fig1}
\end{figure}

Moreover, the figures also display the limits of the decision rules optimising all the numerical performance scores studied in the paper. Here, the decision regions are represented by vertical lines dividing the real number line; the first line to the left divides $\Omega_1$ and $\Omega_2$, and the second line separates $\Omega_2$ and $\Omega_3$, i.e. $\Omega_1=(-\infty, \text{left line}]$, $\Omega_2=(\text{left line}, \text{right line}]$, and $\Omega_3=(\text{right line}, \infty)$. In order to calculate these regions, we directly applied the \textsf{BDR} and the \textsf{EDR} to find their optimal decision regions for $\mathcal{A}cc$ and $\mathcal{A}$, respectively, and an exhaustive search to find the optimal classifiers and their decision regions for the performance scores which have an unknown (to us) decision rule; $\mathcal{G}$, $\mathcal{H}$, $\max$, and $\min$. Since these unknown decision rules have not been named in the manuscript, henceforth, we refer to the unnamed decision rule optimising a score $S$ as $S$\textsf{-DR}, e.g. $\mathcal{G}$\textsf{-DR} will stand for the decision rule optimising $\mathcal{G}$. In the figures, the limiting values are coloured as follows: the red lines represents the limits for the \textsf{EDR}, the orange is for the \textsf{BDR}, yellow is used for $\mathcal{G}$\textsf{-DR}, green for $\mathcal{H}$\textsf{-DR}, and the blue and purple colours are for the extreme rules; $\min$\textsf{-DR} and $\max$\textsf{-DR}, respectively.

By comparing both figures, it can be easily seen that the \textsf{BDR} uses both the class distribution and the feature distribution to split the real number line into regions, and that the \textsf{EDR} relays just on the feature distribution to accomplish the same task. While the \textsf{BDR} cuts the feature space in the intersection points of Figure \ref{fig:unbalanceP}, the \textsf{EDR} relays on the intersections of Figure \ref{fig:balanceP}. Next, note that the insensitive $\max$\textsf{-DR} has only one limiting value (we set this value at $+\infty$, i.e. $\Omega_1 = (-\infty,\infty)$ and $\Omega_2=\Omega_3=\emptyset$). This is due to the fact that the optimality can be easily achieved with the dummy classifier which classifies all instances as just one class. Therefore, an optimal classifier for that rule will be the one that assigns the whole real number line to a determined decision region and that leaves the rest of the decision region empty. Then, regarding the $\min$\textsf{-DR}, it can be seen that it seeks that the area of each class in its decision region is equal to any area of any other class in its decision region so that the minimum recall will be the maximum, i.e.

$$\forall 1\leq i,j \leq 3,  \int\limits_{\Omega_i}f(\mathbf{x}|c_i,\boldsymbol{\theta}_i)d\mathbf{x}=\int\limits_{\Omega_j}f(\mathbf{x}|c_j,\boldsymbol{\theta}_j)d\mathbf{x}.$$

\noindent Finally, the decision regions optimising the other adequate numerical performance scores, i.e. $\mathcal{G}$\textsf{-DR} and $\mathcal{H}$\textsf{-DR} always seem to have limiting values bounded by the \textsf{EDR} and the $\min$\textsf{-DR}. It may be an interesting potential research line to determine how the limiting values of the decision region vary on the class-overlapping for the unweighted H\"{o}lder means ($p \leq 1$). In the figures, it can be seen that while $\Omega_3$ has a similar size for these scores, $\Omega_1$ and $\Omega_2$ fluctuate considerably with $p$.

In conclusion, with this example we reinforce one of the theses of the manuscript. Although, graphically, both the \textsf{BDR} and the \textsf{EDR} seem to be more compelling for splitting the $x$-axis in the intersection points, the use of classifiers maximising unweighted H\"{o}lder means with $p < 1$ may be of interest as we have proved that they are more informative on the class-imbalance problems. More competitive classifiers may be proposed by adding these numerical performance scores to the definition of the forthcoming learning algorithms for unbalanced problems. These adequate classifiers tend to split the $x$-axis so that the decision region of each class shares the same area in terms of just the density function. Therefore, they will be more reactive to the potential problems derived from the skewness of the class distribution.

\end{document}